\documentclass{article}
\usepackage[T1]{fontenc}
\usepackage{iclr2027_conference,times}
\usepackage{amsmath,amsfonts,bm}

\def\eqref#1{equation~\ref{#1}}
\def\1{\bm{1}}

\def\rva{{\mathbf{a}}}

\def\rvg{{\mathbf{g}}}

\def\rvu{{\mathbf{i}}}

\def\rvm{{\mathbf{m}}}

\def\rvo{{\mathbf{o}}}

\def\rvq{{\mathbf{q}}}

\def\rvs{{\mathbf{s}}}

\def\rvu{{\mathbf{u}}}
\def\rvv{{\mathbf{v}}}
\def\rvw{{\mathbf{w}}}
\def\rvx{{\mathbf{x}}}
\def\rvy{{\mathbf{y}}}
\def\rvz{{\mathbf{z}}}

\def\rmD{{\mathbf{D}}}

\def\rmI{{\mathbf{I}}}

\def\rmM{{\mathbf{M}}}

\def\rmR{{\mathbf{R}}}
\def\rmS{{\mathbf{S}}}

\def\rmU{{\mathbf{U}}}
\def\rmV{{\mathbf{V}}}

\def\rmX{{\mathbf{X}}}
\def\rmY{{\mathbf{Y}}}

\DeclareMathAlphabet{\mathsfit}{\encodingdefault}{\sfdefault}{m}{sl}
\SetMathAlphabet{\mathsfit}{bold}{\encodingdefault}{\sfdefault}{bx}{n}

\def\gN{{\mathcal{N}}}

\def\sN{{\mathbb{N}}}

\def\sR{{\mathbb{R}}}
\def\sS{{\mathbb{S}}}

\newcommand{\E}{\mathbb{E}}

\newcommand{\Var}{\mathrm{Var}}

\DeclareMathOperator*{\argmax}{arg\,max}
\DeclareMathOperator*{\argmin}{arg\,min}

\newcommand{\rvmu}{\mathbf{\mu}}

\newcommand{\rmSigma}{\mathbf{\Sigma}}

\usepackage{amsmath,amssymb,amsfonts,amsthm}
\usepackage{mathtools}
\usepackage{bm}
\usepackage{booktabs}
\usepackage{array}
\usepackage{multirow}
\usepackage{graphicx}
\usepackage{subcaption}
\usepackage{float}
\usepackage{hyperref}
\usepackage{url}
\usepackage{dsfont}
\usepackage{xcolor}

\newcommand{\N}{\mathcal{N}}

\newcommand{\diag}{\operatorname{diag}}
\newcommand{\Tra}{\operatorname{tr}}
\newcommand{\SIGReg}{\operatorname{SIGReg}}
\newcommand{\enc}{\operatorname{enc}_\theta}
\newcommand{\pred}{\operatorname{pred}_\phi}

\newcommand{\MN}{ALeWM}
\definecolor{paperpurple}{RGB}{128,0,128}

\definecolor{introgreen}{RGB}{0,128,0}

\newtheorem{proposition}{Proposition}
\newtheorem{corollary}[proposition]{Corollary}
\newtheorem{informalproposition}{Informal proposition}

\title{Adaptive Latent Capacity for World Models}

\author{Idan Achituve, Lior Dikstein, Idit Diamant, Arnon Netzer \& Hai Victor Habi   \\
Arm AI Research (AAIR), Arm Ltd \thanks{ Correspondence to idan.achituve@arm.com}\\
}

\iclrfinalcopy % Uncomment for camera-ready version, but NOT for submission.

\begin{document}
\maketitle

\begin{abstract}
We introduce Adaptive LeWorldModel (\MN{}), a world model based on a joint-embedding predictive architecture (JEPA) that learns to concentrate predictive information in compact prefixes of a wide latent representation. To encourage this ordering, \MN{} learns a sequence-conditioned distribution over prefix lengths and trains the predictor to estimate the full next embedding from a sampled input prefix. As standard anti-collapse objectives
encourage variation across latent coordinates and do not organize them by predictive importance, we also introduce MixSIGReg. MixSIGReg regularizes the masked embeddings against a prior-weighted mixture with Gaussian active prefixes and zeros in the remaining coordinates. As a result, the \MN{} objective encourages early coordinates to retain information useful for prediction and recursive planning. Our analysis shows that the mixture distribution used by MixSIGReg assigns higher variance to earlier coordinate blocks and lower variance to later ones. In addition, we show that, under specified assumptions, prediction error is minimized by placing the information most useful for prediction in earlier blocks. Empirically, we study the behavior of \MN{} in a controlled dynamical system with known state variables and in goal-conditioned visual control. We show that \MN{} consistently achieves higher mean success rates than tuned fixed-width LeWM, with lower planning capacity on average. Project page: \textcolor{magenta}{\url{https://arm-research.github.io/AAIR-ALeWM}} 
\end{abstract}

\section{Introduction}

World models allow agents to anticipate the consequences of actions and use imagined futures to guide behavior \citep{ha2018recurrent, hafner2025dreamerv3}.
Recent research directions include world action models, which jointly generate future observations and actions for robot control \citep{agarwal2025cosmos, kim2026cosmos, li2026causal, ye2026dreamzero}, and explicit 3D world generation and reconstruction of dynamic scenes \citep{hunyuanworld2025,yang2025thinking, zhang2025d4rt}.
Joint-embedding predictive architectures (JEPAs) offer a different route: they predict future embeddings without reconstructing observations \citep{lecun2022path,assran2025vjepa2}.
LeWorldModel (LeWM) makes this approach accessible through a compact encoder and action-conditioned predictor trained jointly from scratch \citep{maes2026leworldmodel}.
Its two-term objective combines next-embedding prediction with SIGReg, which discourages collapse by encouraging Gaussian-distributed embeddings \citep{balestriero2025lejepa}.
This reward-free, end-to-end formulation provides a practical foundation for studying how learned representations support planning.

In latent planning, the learned representation is the state carried through every imagined trajectory. Recent studies of LeWM highlight the importance of how latent states are organized and used for control \citep{li2026rcaux,nguyen2026latentgeometry}. We focus on a complementary question: can a world model retain a wide embedding for learning while concentrating predictive information in a short prefix, consisting of its first few coordinates? Such an ordering would let one trained model support different latent-state budgets without fitting a separate representation for each width. Compact prefixes could offer several benefits: helping planning focus on control-relevant information while limiting the influence of irrelevant visual details, allowing different tasks or episodes to use prefix lengths suited to their information needs, and leaving the remaining coordinates available for features useful to other downstream tasks. Ordered representations already provide this flexibility in reconstruction and retrieval \citep{pmlr-v32-rippel14,kusupati2022matryoshka} using fixed, input-independent weighting over prefix lengths. For a world model, the additional challenge is to make these prefixes useful through repeated action-conditioned predictions and goal comparisons during planning.

Standard anti-collapse objectives do not impose this ordering. SIGReg encourages an isotropic Gaussian embedding distribution, which treats all coordinates symmetrically. VICReg discourages low coordinate variance and cross-coordinate covariance \citep{bardes2021vicreg}. These objectives encourage variation across the representation, but do not assign greater predictive importance to earlier coordinates. One alternative is to reduce the embedding width, but empirically we noticed that this can degrade planning performance. A possible explanation is a tension between representation and prediction: narrow embeddings may restrict the representation of useful scene information, while wider embeddings may distribute predictive information across many coordinates, potentially complicating the dynamics that the predictor must learn. We therefore seek to retain a wide embedding during learning and train a compact prefix to predict the full next embedding, encouraging early coordinates to capture information useful for prediction and planning.

%On alternative is reducing the embedding width; however, it imposes a hard bottleneck on representation learning which can entangle relevant with non-relevant predictive information in the scene. We instead seek to retain a wide embedding while concentrating useful predictive information in its early coordinates. This separates representation width from planning capacity: a wide embedding preserves flexibility during learning, while the prediction objective encourages a compact prefix to retain information useful for predicting the full next embedding.

Here, we introduce \emph{Adaptive LeWorldModel} (\MN{}), a novel method that learns how much of a wide embedding to retain for prediction and planning. Using nested dropout as inspiration \citep{pmlr-v32-rippel14}, we define a capacity network that predicts a distribution over prefix lengths for each observation sequence. During training, the predictor receives a sampled prefix of each history embedding and must predict the full next embedding. This gives the early coordinates a clear role: retain information that helps predict the wider representation even when only a short prefix is available. To accommodate this unequal use of coordinates, we introduce \emph{MixSIGReg}, which regularizes masked embeddings against a prior-weighted mixture of Gaussian prefixes and zero suffixes. At deployment, the capacity network selects a prefix length from the initial and goal observations, fixing it throughout the episode. Planning then uses this prefix for recursive prediction and goal comparison.

\begin{figure}[!t]
  \centering
  \begin{minipage}[c]{0.52\linewidth}
    \centering
    \includegraphics[width=\linewidth]{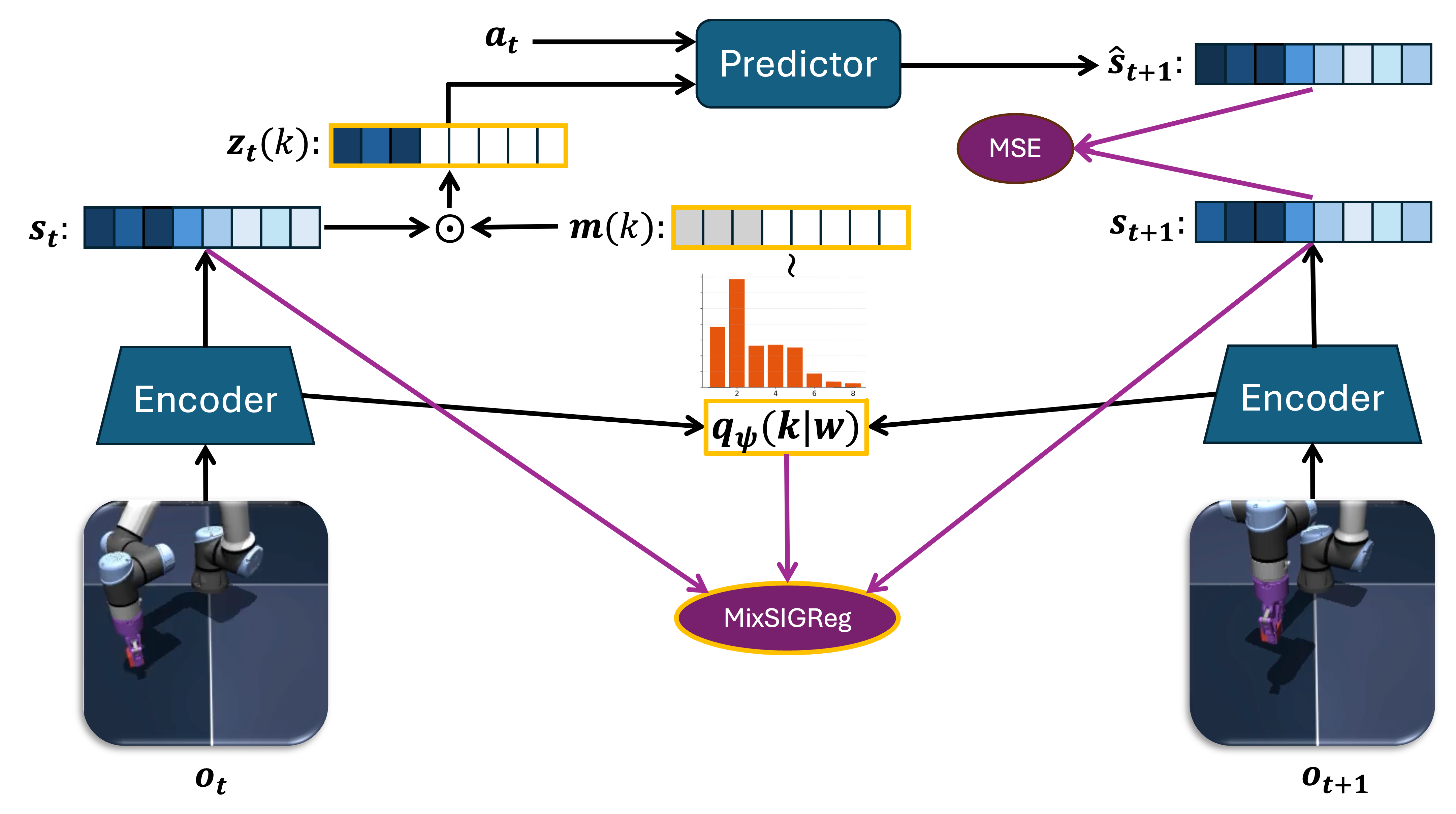}
  \end{minipage}\hfill
  \begin{minipage}[c]{0.46\linewidth}
    \centering
    \includegraphics[width=\linewidth]{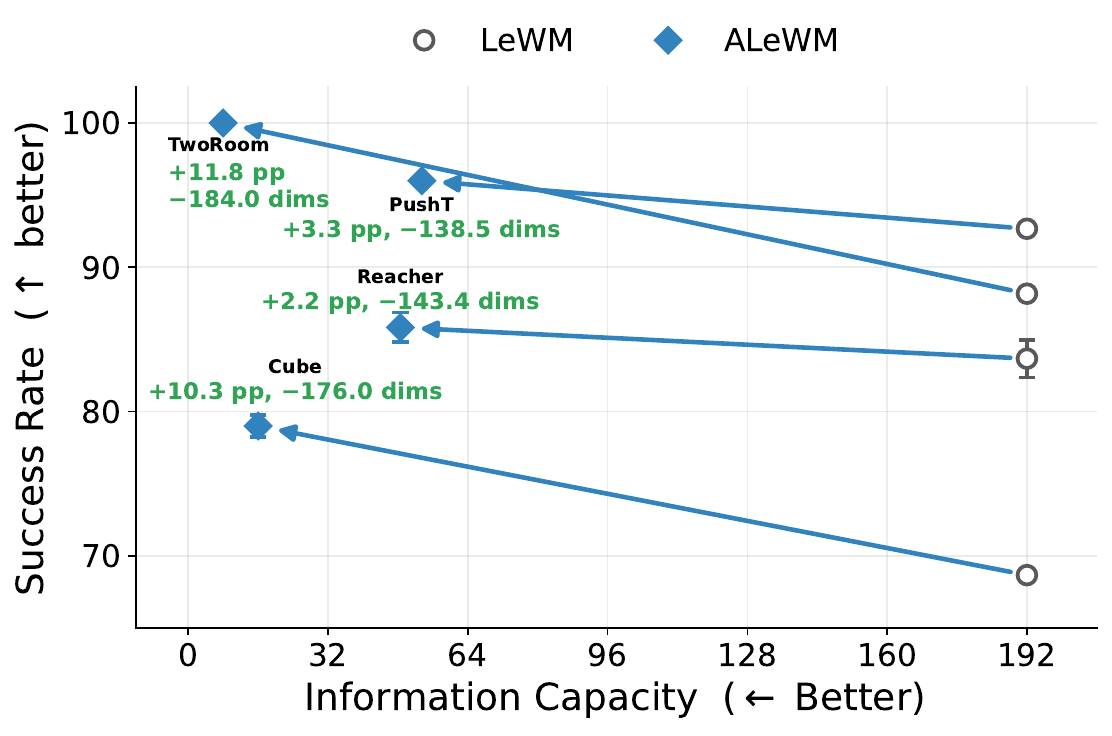}
  \end{minipage}
  \caption{\MN{} overview. \textbf{Left:} A shared encoder maps consecutive observations to latent states, and a learned network selects the active prefix capacity $k$ for masking the state $\rvs_t$. The predictor estimates the full next latent state from the masked embedding and action, while the prediction loss and MixSIGReg objective jointly train the representation. In \textcolor{orange}{orange}, our modifications to the LeWM framework. \textbf{Right:} 
  Gains of \MN{} over fixed-width LeWM using ViT-Tiny models with latent dimension $192$. Moving left and up indicates lower average capacity and higher test success rate.}
  \label{fig:method-overview}
\end{figure}

We analyze \MN{} to explain why early coordinates are favored. The MixSIGReg reference assigns lower variance to later coordinate blocks, reflecting how often they remain active. Our prediction analysis further shows that, under stated assumptions, placing the most useful information first minimizes expected prediction error. Empirically, we study \MN{} in a controlled dynamical system and four common simulated visual control tasks. In the controlled system, an eight-dimensional \MN{} concentrates linearly decodable information about four state variables in its first four coordinates. On the visual control tasks, selected \MN{} models achieve higher mean success rates than tuned fixed-width LeWM, with gains of up to approximately $7$ percentage points (PP) in success rate and an average reduction of $56 \%$ in planning capacity.

To summarize, we make the following novel contributions: (1) We introduce \MN{} with MixSIGReg to concentrate predictive information in compact prefixes while retaining a wide latent representation in JEPA-based architectures; (2) We characterize the mixture target’s decreasing variance across coordinate blocks and establish conditions for ordering predictive information by its contribution to reducing prediction error; (3) We demonstrate concentration of linearly decodable state information in a controlled dynamical system and improved planning success with substantially lower planning capacity on average in goal-conditioned visual control.

\section{Related Work}

\textbf{Joint-embedding world models.}
Joint-embedding predictive architectures (JEPAs) learn representations by predicting related inputs in embedding space rather than reconstructing observations \citep{lecun2022path}. I-JEPA applies this principle to masked image prediction \citep{assran2023ijepa}, while V-JEPA extends feature prediction to video \citep{bardes2024vjepa}. For planning, DINO-WM and V-JEPA~2-AC learn action-conditioned dynamics over pretrained visual features \citep{zhou2024dinowm,assran2025vjepa2}, whereas PLDM, LeWM and Sensorimotor World Models (SMWM) jointly learn visual representations and dynamics from reward-free data
\citep{sobal2025pldm,maes2026leworldmodel,ivashkov2026sensorimotor}. However, predictability alone does not ensure relevance for control: representations may retain persistent distractors \citep{sobal2022slowfeatures}, and low prediction error on the data distribution need not control error on states reached by a planner \citep{you2026controltheory}. Latent-recovery analyses also rely on specific assumptions, including dimension matching and isotropic Gaussian Ornstein--Uhlenbeck dynamics for population LeJEPA \citep{klindt2026when}, or controlled linear-Gaussian dynamics, invertible observations, spectral separation and sufficient conditional action variation \citep{zhang2026identifiabilitycontrolled}. Such results do not establish semantic coordinate ordering for learned prefixes. We build on LeWM's end-to-end setting and LeJEPA's SIGReg machinery \citep{balestriero2025lejepa} to study representation capacity and geometry.

\textbf{Regularizing predictive embeddings.}
Anti-collapse regularizers differ in the properties they impose on the embedding distribution.  Barlow Twins reduces cross-coordinate redundancy \citep{zbontar2021barlow}, W-MSE whitens features \citep{ermolov2021whitening}, and VICReg combines invariance with a coordinate-wise variance floor and an off-diagonal covariance penalty \citep{bardes2021vicreg}. SIGReg matches random projections to a standard Gaussian
\citep{balestriero2025lejepa,zimmermann2025kerjepa}, while VISReg separates scale control from sliced-Wasserstein distribution-shape matching \citep{wu2026visreg}. Other approaches impose subspace constraints or sparse targets: Sub-JEPA regularizes fixed random subspaces \citep{zhao2026subjepa}, and Rectified LpJEPA introduces sparse mixed targets \citep{kuang2026rectifiedlpjepa}. The concurrent study LpWM uses a rectified generalized-Gaussian target to learn non-negative sparse predictive representations \citep{kuang2026lpwm}. LpWM sparsifies a fixed-width code, unlike our construction, which defines nested prefixes.
% For ordered representations, MIC adds prefix--residual correlation, variance and uniformity
% penalties to Matryoshka training \citep{nguyen2026mic}.  Our target has signed Gaussian coordinates in each active prefix and zeros in its suffix, mixed according to a fixed capacity prior.  We match the learned masked aggregate to this target while jointly learning a sequence-conditioned prefix selector.

\textbf{Ordered representations.} Learning ordered representations has been studied in several contexts, including visual tokenization \citep{bachmann2025flextok, yan2025elastictok, yang2025vaportok,lu2026adatok} and representation learning, as depicted next. Nested dropout trains only a sampled prefix, exposing earlier coordinates more often \citep{pmlr-v32-rippel14}. Matryoshka Representation Learning and Information-Ordered Bottlenecks similarly train representations to remain useful across prescribed prefix widths \citep{kusupati2022matryoshka,ho2025iob}. MIC adds prefix--residual correlation, variance and uniformity penalties to Matryoshka training \citep{nguyen2026mic}. Learning the retained capacity also has precedents. Variational Nested Dropout learns input-conditioned ordered masks \citep{cui2023variationalnesteddropout}, while SparC-IB learns input-dependent categorical prefix lengths for supervised prediction \citep{samaddar2023sparcib}. Both use decoder networks to train and organize the latent representation in non-temporal tasks: one reconstructs the input, and the other predicts a label. We instead learn a decoder-free sequence-conditioned distribution over state-coordinate prefix lengths through action-conditioned prediction of future embeddings, and match the aggregate distribution of masked embeddings using characteristic functions.
%Learning the retained capacity also has precedents network: Variational Nested Dropout learns input-conditioned ordered masks through a Bayesian relaxation \citep{cui2023variationalnesteddropout} and SparC-IB learns input-dependent categorical prefix lengths with Gaussian--Dirac support \citep{samaddar2023sparcib}.  It uses per-input active-prefix Kullback--Leibler (KL) divergence and an explicit capacity KL for supervised prediction. Both methods use a decoder which helps organize the latent embedding in a non-dynamical scenarios, while we instead learn a sequence-conditioned categorical prefix over state coordinates through action-conditioned temporal prediction and use characteristic-function matching for the masked aggregate.
%In visual tokenization, FlexTok uses nested dropout to keep ordered prefixes useful across reconstruction and generation budgets \citep{bachmann2025flextok}.
% ElasticTok adapts token allocation for images and videos in both continuous and discrete models, evaluating a learned length predictor alongside
% reconstruction-based search \citep{yan2025elastictok}.  VaporTok and AdaTok learn categorical prefix cutoffs with sampled training and modal inference
% \citep{yang2025vaportok,lu2026adatok}.  These methods vary token counts for visual reconstruction or generation; our prefixes select state coordinates for
% predictive dynamics and planning. 
Lastly, in work concurrent with this study, \citet{liu2026orderedactiontokens} proposed using nested dropout for action-token ordering and reconstruction in vision-language-action models for robot control.

% Nested dropout connects fixed-distribution prefix training to PCA in a linear
% autoencoder, with orthonormality and spectral conditions needed to resolve
% axes up to sign \citep{pmlr-v32-rippel14}.  Matryoshka Representation Learning
% makes selected prefixes useful under different resource budgets
% \citep{kusupati2022matryoshka}.  Variational Nested Dropout already includes
% learned, input-conditioned masks and Gumbel--Softmax relaxation
% \citep{cui2023variationalnesteddropout}.  FlexTok and ElasticTok extend
% variable-length representations to visual tokenization
% \citep{bachmann2025flextok,yan2025elastictok}.
% We extend established nested-capacity ingredients to an action-conditioned
% JEPA, matching its anti-collapse target to the masked aggregate and reusing
% the selected prefix recursively during planning.

% \section{Background}
% Scalars are lower-case, vectors bold lower-case, and matrices bold upper-case.
% The maximum embedding dimension is $d$.  LeWM encodes observations as
% $\rvs_t=\enc(\rvo_t)$ and predicts the next embedding using an
%  action-conditioned history of embeddings.  Training combines squared
% prediction error with SIGReg, which compares empirical characteristic
% functions along random projections with a Gaussian reference.  Planning uses
% CEM to optimize actions through latent rollouts and a terminal goal distance.
% Appendix~\ref{sec:background-details} gives the full equations and sampling
% conventions.  We change the predictor's input state and the SIGReg reference;
% the full next embedding remains the prediction target.

\section{Method}
\label{sec:method}
%Scalars are lower-case (e.g., $x$), vectors are bold lower-case (e.g., $\rvx$), and matrices are bold upper-case (e.g., $\rmX$).  %Throughout, $d$
%is the maximum embedding dimension, $\rvs_t\in\sR^d$ denotes a full encoder
%embedding, and $\rvz_t\in\sR^d$ denotes its capacity-masked version. 
%In what follows, we describe \MN{}'s objective and planning, and provide theoretical justification for our proposed approach.
To motivate our method, we begin with the LeWM formulation \citep{maes2026leworldmodel}, which combines two objectives: (1) a prediction loss between the predicted next-step embedding and the encoder's embedding of the corresponding observation; and (2) SIGReg regularization
\citep{balestriero2025lejepa}, which encourages the aggregate embedding distribution to match an isotropic standard Gaussian (see Appendix~\ref{sec:background-details} for a full description).
These objectives do not impose a coordinate ordering, allowing predictive information to be distributed throughout the latent representation. Empirically, we observed that increasing the embedding dimension can often reduce planning performance. One possible explanation is that predicting a larger latent can make accurate prediction more difficult. Since during rollout the learned predictor is applied recursively, prediction errors can propagate and affect the planner's evaluation of candidate actions. Reducing the embedding dimension can mitigate this issue, but requires careful tuning and restricts the expressivity of the latent representation.

Hence, we seek to retain the expressivity of a wide embedding while encouraging the information needed for prediction to concentrate in compact prefixes, resulting in reduced difficulty of prediction. Inspired by nested dropout \citep{pmlr-v32-rippel14}, which was proposed in the context of autoencoders, we learn a sequence-conditioned distribution over a finite set of admissible prefix capacities. During training, the predictor receives a sampled prefix of the latent embedding vectors and predicts the full next embedding, encouraging early coordinates to retain information useful across multiple capacities. The prefix size is sampled from a distribution produced by a capacity network which is jointly \textit{learned} with the rest of the network parameters, namely, the encoder and the predictor. At deployment, the capacity network selects a fixed prefix size for each episode based on the known initial and goal frames. Our construction displays several favorable properties: (1) As the capacity network is learned during training, the prefix size can be set automatically without an exhaustive search after training; (2) different episodes can use different capacities without training separate models; and (3) alternative prefix sizes can be examined for each episode after training. A graphical description of our method is presented in Figure~\ref{fig:method-overview}. We describe our construction in detail next.

\subsection{Controlling the Information Capacity}
In what follows, scalars are denoted in lower-case (e.g., $x$), vectors in bold lower-case (e.g., $\rvx$), and matrices in bold upper-case (e.g., $\rmX$).
Assume we are given a training sequence of $F$ observations $\rvw = (\rvo_1,\ldots,\rvo_F)$, and a set of associated actions $(\rva_1, \ldots, \rva_{F-1})$. We define the set of admissible capacities as an ordered set $\mathcal{K}=\{k_1<\cdots<k_C\}\subseteq\{1,\ldots,d\}$ with $k_C=d$, the maximal latent embedding dimension, and $k_c \in {\sN} \setminus \{0\} ~~ \forall c \in \{1, \ldots, C\}$. For instance, if we fix $C=d$, then $\{k_1=1< k_2=2 < \cdots<k_C=d\}$. Alternatively, if we fix $C=d/2$, a possible capacity set is the set of even numbers $\{k_1=2< k_2=4 < \cdots<k_C=d\}$, assuming that $d$ is even. 
In our experiments, the capacity set is fixed for each $d$ (Appendix \ref{sec:implementation_details} specifies the choices per experiment).

Given $k\in\mathcal K$, define the prefix mask by $\rvm(k)$ with coordinate values $[\rvm(k)]^j = \mathds{1}[j\le k], ~~ \forall j \in \{1,\ldots,d\}$. The latent embedding $\rvs_t \in \sR^d$ and the masked-latent embedding $\rvz_t(k) \in \sR^d$ for a time-step $t$ are obtained according to:
\begin{equation}
    \rvs_t=\enc(\rvo_t),
    \qquad
    \rvz_t(k)=\rvm(k)\odot\rvs_t=\rmM_k\rvs_t,
    \qquad
    z_t^j(k)=
    \begin{cases}
        s_t^j, & j\leq k,\\
        0, & j>k,
    \end{cases}
  \label{eq:masked-latent}
\end{equation}
%where $\rmM_k=[\rvm(k_1), ..., \rvm(k_C)] \in \{0,1\}^{C \times d}$ is the matrix formed by stacking the mask vectors in columns, 
where $\rmM_k=\diag(\rvm(k)) \in \{0,1\}^{d \times d}$ and $\enc(\cdot)$ is an encoder network. Next, we define a \textit{learned} capacity selector network $q_\psi(k \mid \rvw)$ and a \textit{fixed} capacity prior distribution $\pi_0(k; \alpha)$ whose role will become clear in the MixSIGReg regularizer. Both distributions lie in $\Delta^{C-1}$. We construct the capacity prior from polynomial weights over reverse capacity ranks. For degree $\alpha\in\sR$, treated as a hyperparameter, define
\begin{equation}
  w_\alpha(k_c)=(C-c+1)^\alpha,
  \qquad
  \pi_0(k_c;\alpha)
  =
  \frac{w_\alpha(k_c)}{\sum_{\ell=1}^{C}w_\alpha(k_\ell)}.
  \label{eq:prior-family-cat}
\end{equation}
Thus, $\alpha>0$ favors smaller capacities, $\alpha=0$ gives a uniform prior, and $\alpha<0$ favors larger capacities.  In this study, our construction depends only on the ordering of the capacity support, not on the numerical spacing between its
elements. We leave for future work further investigations of other weighting schemes for the capacities. To learn the capacity-network parameters, we initialize the network such that its output is equal to prior probabilities $\boldsymbol{\pi}_0(\cdot, \alpha)$ and use the straight-through Gumbel--Softmax estimator: we draw $k\sim q_\psi(\cdot\mid\rvw)$, use the corresponding hard prefix in the forward pass, and differentiate through its soft relaxation \citep{jang2017gumbelsoftmax, maddison2017the} (full construction in
Appendix~\ref{sec:st_estimator}). We intentionally condition the capacity on a sequence of observations to constrain the capacity to be the same for the entire sequence while allowing the model to choose different capacities for different sequences. Intuitively, for the same sequence we expect that the same information capacity will store all the relevant details, yet it may vary between sequences.

\subsection{Training Objective and Planning} As in LeWM, our objective consists of prediction loss for the next state given a window of previous states, and a regularization term that prevents embedding collapse. For clarity, we present next-step prediction; the objective extends naturally to sequences of frames and actions with frame skips.

\paragraph{Prediction loss.} Let $\rvs_t=\enc(\rvo_t)$ and $\rvs_{t+1}=\enc(\rvo_{t+1})$ be the encoding of the current and next states, respectively. In addition, let $\rva_{t}$ be the current (ground-truth) action and let $k\sim q_\psi(\cdot\mid\rvw)$ be the sampled capacity with its corresponding mask $\rmM_k$. The model predicts the next state as follows:
\begin{equation}
    \rvz_{t}(k)=\rmM_k\rvs_{t},
    \qquad
    \hat{\rvs}_{t+1} = \pred(\rvz_{t}(k), \rva_t), 
    \label{eq:prediction}
\end{equation}
where $\pred(\cdot, \cdot)$ is a predictor network. We define the prediction target to be the full next encoder embedding $\rvs_{t+1}$. Thus, the prediction loss is given by 
\begin{equation}
    \mathcal{L}_{\mathrm{pred}}
    =
    \frac{1}{B(F-1)}
    \sum_{b=1}^B \sum_{t=1}^{F-1}|| \hat{\rvs}_{b, t+1} - \rvs_{b, t+1}||_2^2 \,
    \label{eq:prior-marginal-prediction-loss}
\end{equation}
where $B$ is the batch size, and the subscript $b$ denotes the index of an example in a batch.
Note that in the objective both the input and target embeddings are learned jointly without relying on a fixed external teacher. The intuition behind predicting the full next embedding is to encourage the prefix to retain information useful for predicting the wider representation, while preserving flexibility during representation learning. Since earlier coordinates are retained at least as often as later ones, the model is encouraged to concentrate predictive information near the beginning of the embedding.

\paragraph{MixSIGReg.}
Applying the prediction loss alone can lead to representation collapse. Common regularization schemes like 
SIGReg \citep{balestriero2025lejepa, maes2026leworldmodel}
and VICReg
\citep{bardes2021vicreg} encourage non-degenerate variance uniformly across all coordinates. Applying this constraint to the masked embedding may not be appropriate in our case, as we require later coordinates to be active less often and to be zero when masked. 
Indeed, in Section~\ref{sec:analysis}, we empirically observe that the SIGReg loss does not enforce a small prefix size for the masked embedding. To account for this, we define a reference distribution for each capacity $k$, with a standard Gaussian on the active coordinates and zeros on the remaining coordinates: $
p_0(\rvz_t \mid k)
=
\gN(\rvz_t^{1:k};0,\rmI_k)
\otimes
\delta_0(\rvz_t^{k+1:d})$, 
where $\delta_0$ is the Dirac distribution centered at zero, $\rvz_t^{1:k}$ denotes the first $k$ coordinates of $\rvz_t$,
and $\rvz_t^{k+1:d}$ denotes its remaining $d-k$ coordinates.

Let $\pi_0(k;\alpha)$ be the polynomial capacity prior in
Eq.~\ref{eq:prior-family-cat}. The marginal reference distribution of $\rvz_t$ takes the form of a mixture of Gaussian distributions:
\begin{equation}
    p_0(\rvz_t)
    =
    \sum_{k\in\mathcal{K}}\pi_0(k;\alpha)
    \left[
        \gN(\rvz_t^{1:k};0,\rmI_k)
        \otimes
        \delta_{0}(\rvz_t^{k+1:d})
    \right].
    \label{eq:prior-marginal-target}
\end{equation}

Having defined the reference distribution, we follow the SIGReg recipe by projecting embeddings onto unit-norm directions and minimizing a univariate Epps–Pulley \citep{epps1983normality} test statistic adapted to this reference distribution. Specifically, we draw $P$ projection vectors $\{\rvu^{(p)}\}_{p=1}^P$, with each $ \rvu^{(p)}\in \sS^{d-1}$. Under this reference law, $\rvu^{(p)\top}\rvz_t\mid K=k
    \sim
    \gN\!\left(
        0,
        \rvu^{(p)\top}\rmM_k\rvu^{(p)}
    \right)$.
Therefore, for a particular direction $p$ the target characteristic function at an evaluation point $\tau$ is:
\begin{equation}
    \phi_0^{(p)}(\tau)
    =
    \sum_{k\in\mathcal{K}}\pi_0(k;\alpha)
    \exp\!\left(
        -\frac{\tau^2}{2}
        \left\|\rmM_k\rvu^{(p)}\right\|_2^2
    \right).
    \label{eq:prior-marginal-target-cf}
\end{equation}

Likewise, we estimate the characteristic function at frame position $t$ by marginalizing over prefix capacities using the learned capacity distribution $q_\psi(k\mid\rvw_b)$ and applying a Monte Carlo average over batch observations:
\begin{equation}
    \hat{\phi}_t^{(p)}(\tau)
    =
    \frac{1}{B}
    \sum_{b=1}^{B}\sum_{k\in\mathcal{K}}
    q_\psi(k\mid\rvw_b)
    \exp\!\left(
        i\tau\rvu^{(p)\top}\rmM_k\rvs_{b,t}
    \right).
    \label{eq:prior-marginal-empirical-cf}
\end{equation}
The resulting loss, MixSIGReg, is obtained by averaging over projections and frame positions:
\begin{equation}
    \mathcal{L}_{\mathrm{MixSIG}}
    =
    \frac{B}{PF}
    \sum_{p=1}^{P}\sum_{t=1}^{F}
    \int
    w(\tau)
    \left|
        \hat{\phi}_t^{(p)}(\tau)-\phi_0^{(p)}(\tau)
    \right|^2
    d\tau.
    \label{eq:prior-marginal-sigreg-loss}
\end{equation}
Here $w(\tau)=\exp(-\tau^2/2)$ is the Gaussian frequency window and the integral is approximated using the quadrature specified in Appendix~\ref{sec:main-results-implementation}. Overall, MixSIGReg combines collapse prevention with a preference for compact representations. Its reference distribution assigns more frequent use to earlier coordinates while requiring active prefixes to retain variation.

To summarize, the complete training objective of ALeWM is:
\begin{equation}
    \mathcal{L}
        =
        \mathcal{L}_{\mathrm{pred}}
        +
        \lambda \mathcal{L}_{\mathrm{MixSIG}}
    \label{eq:prior-marginal-complete-objective}
\end{equation}

\paragraph{Fixed-prefix rollout.} As in LeWM, \MN{} performs planning using model predictive control (MPC). Unlike LeWM,
however, a single model supports planning with multiple prefix capacities. Given an episode's initial observation $\rvo_1$ and goal observation $\rvo_g$, we select the modal supported capacity $k_r=\argmax_{k\in\mathcal K}q_\psi(k\mid\rvw)$ where $\rvw=(\rvo_1,\rvo_g)$, and hold it fixed throughout the episode. Candidate action sequences are optimized so that the predicted terminal latent state at planning horizon $H$ approaches the goal embedding in the active-prefix space. Specifically, the same prefix mask $\rmM_{k_r}$ is used across all recursive rollout steps and replans, as well as in the terminal goal cost. Other deployment rules are possible, such as sampling a capacity from $q_\psi$ or evaluating multiple supported capacities per episode according to a specified
selection criterion, yet we did not pursue these directions in this study. Concretely, 
\begin{equation}
    \hat{\rvz}_{1}=\rmM_{k_r}\enc(\rvo_{1}),
    \qquad
    \hat{\rvz}_{t+1}
    =
    \rmM_{k_r}\pred(\hat{\rvz}_{t},\rva_{t}).
    \label{eq:capacity-conditioned-rollout}
\end{equation}
Let $\rvs_{g}$ denote the full encoder embedding of the goal
observation. Candidate plans are scored only on the active target dimensions:
\begin{equation}
    \mathcal{C}(\rva_{b,1:H-1};k_r)
    =
    \frac{1}{k_r}
    \left\|
        \hat{\rvz}_{H}-\rmM_{k_r}\rvs_g
    \right\|_2^2.
    \label{eq:active-prefix-goal-cost}
\end{equation}
The optimal candidate plan is then selected based on 
$\argmin_{\rva_{1:H-1}}\mathcal{C}(\rva_{1:H-1};k_r)$ using the Cross-Entropy Method (CEM) \citep{rubinstein2004cross}.

\subsection{Theoretical Motivation For Ordered Representation}
\label{sec:theory}
Our analysis separates two effects.  The MixSIGReg reference prescribes an ordered variance profile for masked embeddings, while nested prediction rewards information that remains useful across multiple prefix sizes.

\begin{informalproposition}[Ordered variance and predictive value]
\label{prop:informal-ordered-predictive-value}
Let $K_\pi\sim\pi_0$ be independent of the data and define
$\rho_\pi^j=\Pr(K_\pi\ge j)$.
The MixSIGReg reference satisfies
\[
\Var(z_0^j)=\rho_\pi^j,
\qquad
\rho_\pi^1\ge\cdots\ge\rho_\pi^d.
\]
Exact population matching transfers this variance profile to the
learned masked embeddings.

For a fixed encoder with square-integrable targets, let $R_k$
be the minimum population squared error for predicting the full
next embedding from the first $k$ coordinates of the embedding
history and the actions. Nested information gives
$R_\ell\le R_k$ for $k\le\ell$.
Setting $k_0=0$ and
$\Delta_c=R_{k_{c-1}}-R_{k_c}\ge0$, we obtain
\[
\E_{K_\pi}[R_{K_\pi}]
=R_0-\sum_{c=1}^C\rho_\pi^{k_c}\Delta_c.
\]
Each predictive gain is therefore weighted by its block's
retention probability. If these gains are additive and unchanged
by reordering, placing larger gains earlier minimizes this risk.
\end{informalproposition}

Formal assumptions and proofs appear in Appendix~\ref{sec:theory-details}. We note that these results motivate the proposed ordering but do not guarantee that joint training achieves it.

\paragraph{Intuition.}
One can think of a prefix mask as a cutoff.  Coordinate $j$ appears only when the cutoff reaches $j$.  In the reference distribution it has unit variance when present and is zero otherwise, so its overall masked variance is simply its survival probability $\rho_\pi^j$.  Any cutoff that keeps a later coordinate also keeps all earlier ones; this makes the masked variances non-increasing. Nested prediction supplies the complementary pressure.  For a fixed representation, an ideal predictor with a longer prefix cannot have larger minimum squared error because it has all the information in every shorter prefix.  The decrease $\Delta_c$ measures the predictive value added by block
$c$.  Under a random cutoff, that value is available only when the block survives, so early predictive information is reused across more capacity choices.  To see the resulting ordering pressure, suppose two blocks carry fixed predictive gains.  Placing the larger gain first makes its reduction in error count for more sampled capacities while leaving it later makes it count only when a longer prefix is selected.  Exchanging a larger late gain with a smaller early one therefore lowers expected prediction error whenever the two positions have different survival probabilities.  The preference therefore is for larger reductions in risk, not for larger remaining risks, near the start.  Hence, if later blocks add little oracle predictive value, a short prefix is nearly as predictive as the full embedding.

% These statements describe an inductive bias, not a guarantee about joint
% training.  The implemented finite MixSIGReg loss only encourages population
% matching, the risks $R_k$ use a separate Bayes-optimal predictor at each prefix,
% and learned-selector loss changes can be negative for a shared finite model.
% The results therefore do not establish convergence, semantic factor
% identification, or recovery of an intrinsic state dimension.

% Exact population matching passes this pattern to the learned masked embedding.
% It does not guarantee the same ordering for the raw embedding.  When a
% coordinate is inactive, its raw encoder value is invisible to the masked
% matching constraint; variation in those inactive values can therefore change
% the raw variances and their order.  The prediction loss and shared model may
% still constrain these values, but the target-law equality does not order them.

% \paragraph{Operational value and its limits.}
% A packed prefix needs $K$ scalars and its squared goal distance uses $K$
% coordinate terms, compared with $d$ for the full embedding.  This motivates
% compact rollout buffers, stored goal/replay embeddings, and staged candidate
% comparison.  These uses require suitable packing or computation; the current
% implementation masks dense tensors and does not establish encoder or predictor
% speedups.  Moreover, prediction quality is task- and loss-dependent: a
% persistent distractor can be highly predictable.  Prefix control curves are
% therefore necessary evidence beyond the prediction objective.

\section{Experiments}
We evaluated \MN{} on a toy problem (section~\ref{sec:toy}) as well as on a diverse set of simulated 2D and 3D tasks (section~\ref{sec:main_results}).  Unless specified otherwise, we report the average performance and standard error of the mean (SEM) over $3$ random seeds. %on the test set using validation-based selection.
Full experimental details are given in Appendix~\ref{sec:implementation_details}.

\begin{figure}[t]
  \centering
  \begin{minipage}[t]{0.235\linewidth}
    \centering
    \includegraphics[width=\linewidth]{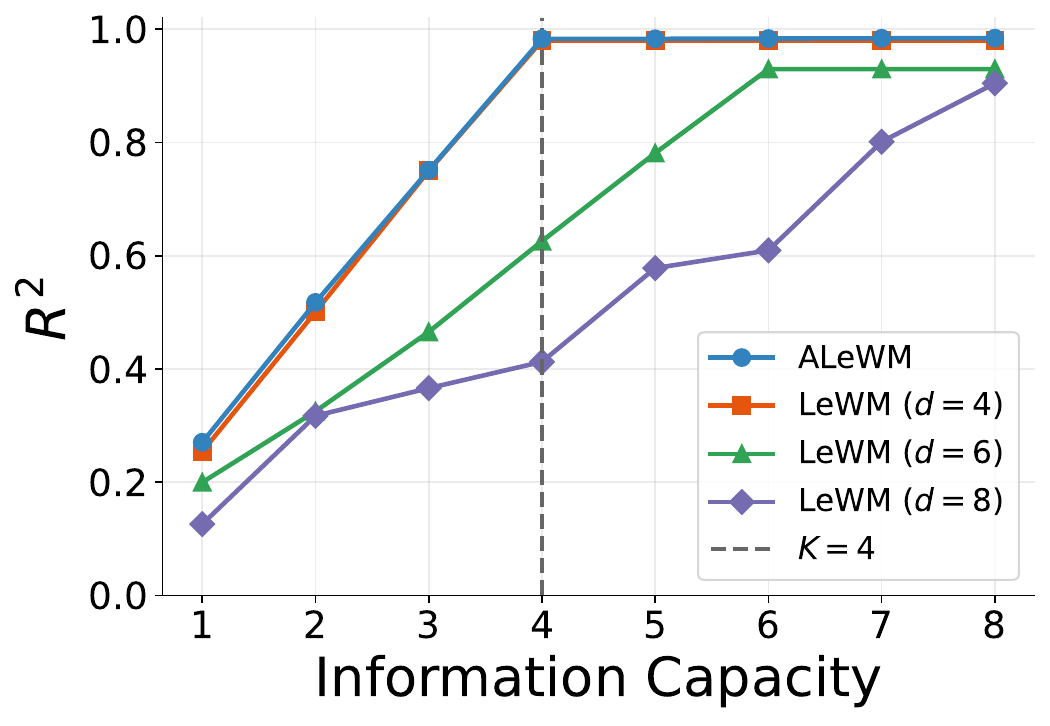}\\[-1mm]
    \small (a) Linear state recovery.
  \end{minipage}
  \hfill
  \begin{minipage}[t]{0.235\linewidth}
    \centering
    \includegraphics[width=\linewidth]{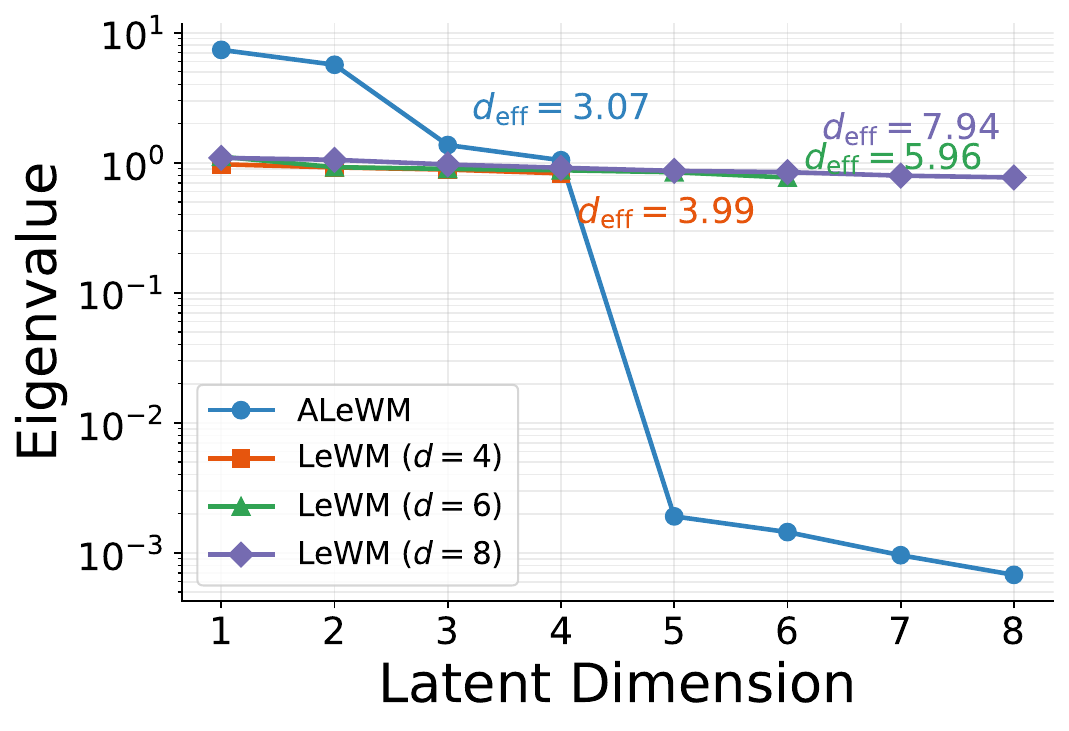}\\[-1mm]
    \small (b) Latent covariance spectrum.
  \end{minipage}
  \hfill
  \begin{minipage}[t]{0.235\linewidth}
    \centering
    \includegraphics[width=\linewidth]{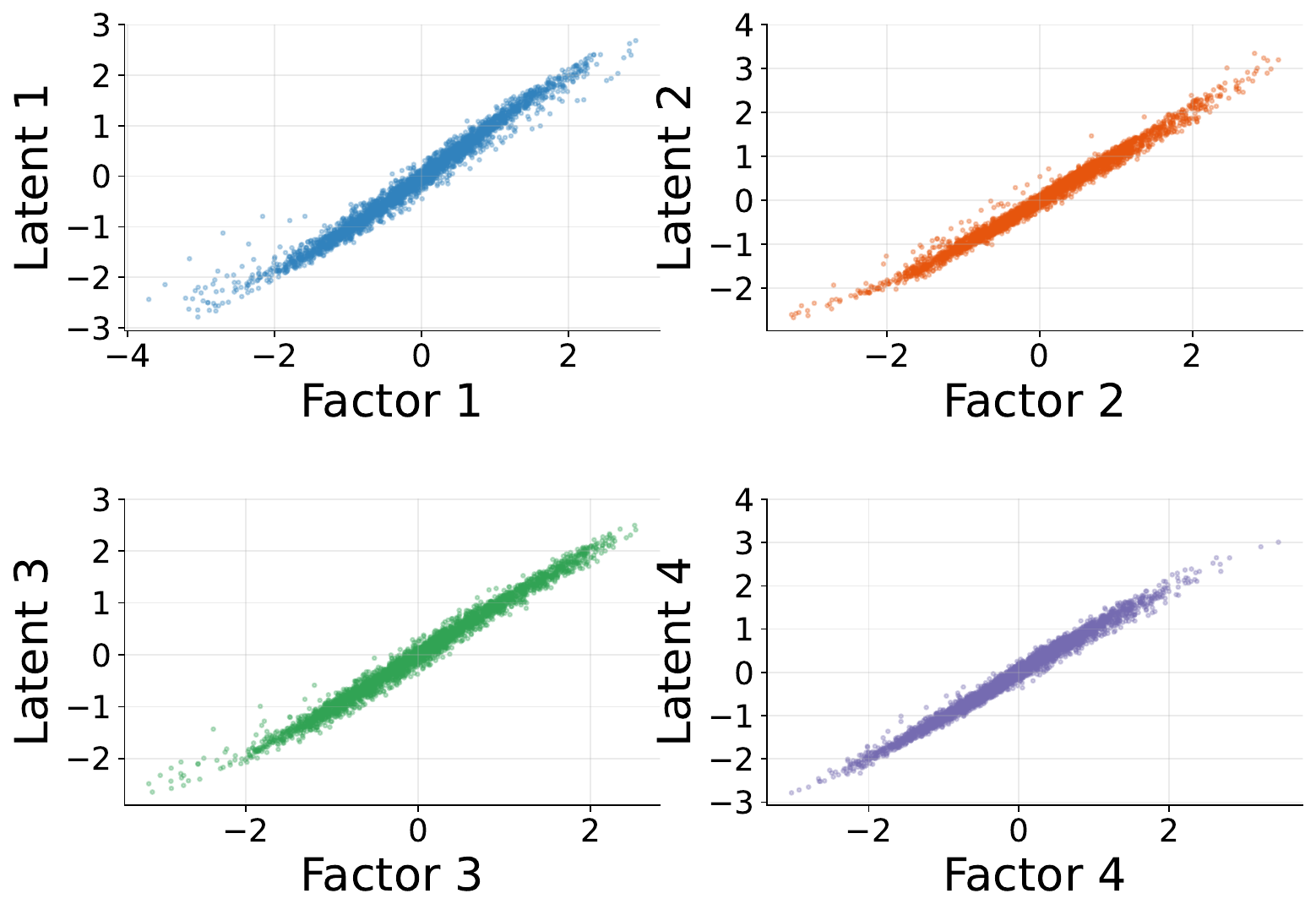}\\[-1mm]
    \small (c) Whitened Procrustes alignment.
  \end{minipage}
  \hfill
  \begin{minipage}[t]{0.235\linewidth}
    \centering
    \includegraphics[width=\linewidth]{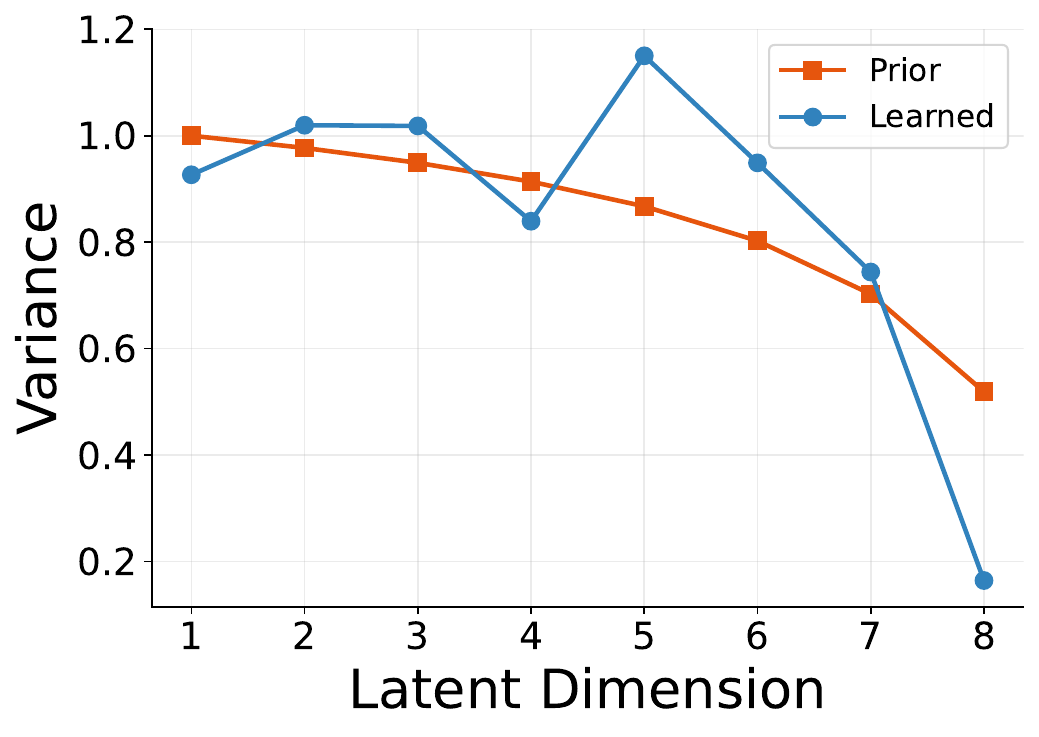}\\[-1mm]
    \small (d) Masked embedding variance.
  \end{minipage}
  \caption{
  Four-factor controlled damped oscillator diagnostics.  \textbf{(a)} Mean linear state-recovery $R^2$ as prefix size increases; the dashed line marks the true factor count $k=4$. For LeWM with $d \in \{4, 6\}$ the curves are continued horizontally beyond their trained dimension only as a visual reference. \textbf{(b)} Raw unmasked latent covariance spectra. \textbf{(c)} \MN{} whitened Procrustes alignment, fitted using training statistics. \textbf{(d)} \MN{} empirical masked-latent variance vs the target prior variance.}
  \label{fig:toy-oscillator}
\end{figure}

\subsection{Toy Example - controlled damped oscillators}
\label{sec:toy}
To gain insight into how \MN{} organizes dynamically relevant information across its latent coordinates, we study a controlled system with a known four-dimensional state. This setting allows us to test whether \MN{} can concentrate linearly decodable state information in an early prefix while retaining a wider latent representation. The system contains two independently controlled damped oscillators with state $\rvx_t=(p_t^1,v_t^1,p_t^2,v_t^2)$, corresponding to the oscillators’ positions and velocities. We record trajectories containing eight transitions. The actions $(a_t^1,a_t^2)$ specify the two oscillators’ equilibrium positions and are switched after four transitions. To form each observation, an invertible nonlinear map produces four strong observation coordinates, while six noisy nonlinear measurements mix all four factors along fixed random directions. Thus, each observation is ten-dimensional but contains only four dynamical factors: four coordinates are strongly related to the factors, while the remaining six provide weaker, noisier measurements of them. We train an eight-dimensional \MN{} with capacity support
$\mathcal{K} = \{1,\ldots,8\}$ and compare it with separately trained LeWM models with $d\in\{1, \ldots, 8\}$ latent dimensions. The results below are shown for one seed, yet they were consistent across several seeds.
Appendix ~\ref{sec:toy-implementation}  provides full experimental details.

% Figure~\ref{fig:toy-oscillator} presents several diagnostics of \MN{}, comparing them with LeWM (with latent dimension in \{4, 6, 8\}) where applicable on a held-out test set. For both methods, the reported results correspond to the hyperparameter-sweep configurations with the highest validation-set mean linear state-recovery \(R^2\). Figure~\ref{fig:toy-oscillator}(a) supports that when the latent dimension and the number of latent factors match LeWM linearly recovers the world’s latent variables \citp{klindt2026does}; however, when the latent dimension is larger it is harder for LeWM to achieve linear identifiably. We witnessed this pattern across myriad of weights for the SIGReg. Figure~\ref{fig:toy-oscillator}(b) further supports that finding showing that LeWM remains nearly isotropic with the effective ranks approximately equal to the number of latent dimensions for all dimensions. On the other hand, Figure~\ref{fig:toy-oscillator}(a)-(c) also show that despite the latent dimension is set to $d=8$, \MN{} can concentrate most of the information at prefix of $4$ with $R^2$ of $0.983$ (versus $0.980$ for LeWM with $d=4$), obtain anisotropic spectrum with $4$ non degenerate ranks, and obtain high alignment between the latent factors and the first four coordinates of the latent dimension. We stress that results may vary depending on the prior polynomial degree and $\lambda$. 

Figure~\ref{fig:toy-oscillator} compares latent-structure diagnostics for \MN{} and fixed-width LeWM models with $d\in\{4,6,8\}$ on the held-out test set. LeWM attains $R^2=0.980$ when its dimension matches the four-dimensional state, but recovery falls to $0.930$ and $0.904$ at $d=6$ and $d=8$, respectively (Figure~\ref{fig:toy-oscillator}(a)). This result is consistent with \citet{klindt2026when}, whose identifiability guarantee assumes matching representation and latent dimensions and leaves the mismatched regime open. Moreover, LeWM remains nearly isotropic at every width %, with effective ranks $3.99$, $5.96$, and $7.94$ 
(Figure~\ref{fig:toy-oscillator}(b)).
In contrast, Figure~\ref{fig:toy-oscillator}(a)--(c) show that \MN{} resolves this width--prefix tradeoff with non-degenerate prior variance for all capacities (Figure~\ref{fig:toy-oscillator}(d)). 
%compares the empirical variance of the sampled masked embedding with the prior survival profile, which is nonzero in all eight coordinates. These quantities would coincide under exact population matching on the evaluated distribution; the observed deviations can arise from finite test sampling, the finite training surrogate, and incomplete optimization. 
Although \MN{} was trained with maximum dimension $d=8$, its first four coordinates attain $R^2=0.983$, with minor improvement from the remaining coordinates. This matches the dimension-matched LeWM model while outperforming the selected wider models. \MN{} also exhibits four dominant covariance directions and close alignment between its four-coordinate prefix and the four factors after fitting an orthogonal alignment matrix. Together, these results show that \MN{} can concentrate linearly decodable state information in a compact prefix without restricting the ambient latent width. We note that these results support prefix concentration; however, coordinate-wise disentanglement or recovery of the true factor count is not guaranteed.

\subsection{Main Results}
\label{sec:main_results}
Next, we evaluate goal-conditioned control on the four continuous-action benchmarks used by LeWM \citep{maes2026leworldmodel}: TwoRoom, PushT, Reacher, and OGBench-Cube. For each dataset, we construct a new episode-level split with no overlap among the training, validation, and test episodes, thereby preventing data leakage. The validation and test sets contain $100$ and $200$ episodes, respectively, and the remaining episodes are used for training. Validation set is used for hyper-parameter selection in all experiments throughout. We follow the training and planning protocol of LeWM. In addition, we use a small transformer network for $q_\psi(k \mid \rvw)$ that supports variable-length sequences and is permutation-invariant over its input tokens. We consider ViT-Tiny encoders with maximum latent dimensions $d_{\max}\in \{96,192\}$ and a ViT-Small encoder with $d_{\max}=384$. Table~\ref{tab:main-control-results} reports the test success rates of \MN{} and two LeWM variants: one with the latent dimension set to $d_{\max}$ and one with the best latent dimension searched as a hyper-parameter. For \MN{}, we use a coarse capacity support for each $d_{\max}$; for example, when $d_{\max}=192$, the capacity support is $\mathcal{K}=\{8,16,32,64,96,128,160,192\}$. %The evaluated LeWM search grids are similar and specified in Appendix~\ref{sec:main-results-implementation}. 
This coarser support reduces the categorical search space and allows each capacity increment to activate a nontrivial block of coordinates, making capacity allocation easier to optimize. For each test episode, planning uses the capacity with the highest selector probability, denoted by $K_{\text{plan}}$. Table~\ref{tab:main-control-results} reports the average of this selected capacity over $200$ test episodes within each seed and then over the $3$ seeds. Full experimental details are 
in Appendix~\ref{sec:main-results-implementation}.

\begin{table}[!t]
  \caption{LeWM and \MN{} mean test success rate ($\pm$ SEM) and mean planning prefix $\E[K_{\text{plan}}]$.}
  \label{tab:main-control-results}
  \centering
  \setlength{\tabcolsep}{3pt}
  \scalebox{1.0}{%
  \resizebox{\linewidth}{!}{%
  \begin{tabular}{ll*{8}{r}}
    \toprule
    Backbone ($d_{\max}$) & Method
      & \multicolumn{2}{c}{TwoRoom}
      & \multicolumn{2}{c}{PushT}
      & \multicolumn{2}{c}{Reacher}
      & \multicolumn{2}{c}{OGBench-Cube} \\
      & & SR ($\uparrow$) & $\E[K_{\text{plan}}]$
      & SR ($\uparrow$) & $\E[K_{\text{plan}}]$
      & SR ($\uparrow$) & $\E[K_{\text{plan}}]$
      & SR ($\uparrow$) & $\E[K_{\text{plan}}]$ \\
    \midrule
    ViT-Tiny (96) & LeWM (best $d$)
      & $95.00\pm0.00$ & $32.00$
      & $89.33\pm0.44$ & $96.00$
      & $83.67\pm0.73$ & $64.00$
      & $72.33\pm1.17$ & $96.00$ \\
                  & LeWM ($d=96$)
      & $92.17\pm0.17$ & $96.00$
      & $89.33\pm0.44$ & $96.00$
      & $82.00\pm1.76$ & $96.00$
      & $72.33\pm1.17$ & $96.00$ \\
                  & \MN{}
      & $\mathbf{99.50\pm0.00}$ & $13.33$
      & $\mathbf{94.67\pm0.83}$ & $32.00$
      & $\mathbf{85.83\pm0.60}$ & $21.33$
      & $\mathbf{77.17\pm1.17}$ & $16.00$ \\
    \midrule
    ViT-Tiny (192) & LeWM (best $d$)
      & $95.00\pm0.00$ & $32.00$
      & $92.67\pm0.33$ & $160.0$
      & $84.50\pm1.15$ & $160.0$
      & $72.33\pm1.17$ & $96.00$ \\
                   & LeWM ($d=192$)
      & $88.17\pm0.33$ & $192.0$
      & $92.67\pm0.33$ & $192.0$
      & $83.67\pm1.30$ & $192.0$
      & $68.67\pm0.17$ & $192.0$ \\
                   & \MN{}
      & $\mathbf{100.0\pm0.00}$ & $8.00$
      & $\mathbf{96.00\pm0.00}$ & $53.49$
      & $\mathbf{85.83\pm1.01}$ & $48.59$
      & $\mathbf{79.00\pm0.76}$ & $16.00$ \\
    \midrule
    ViT-Small (384) & LeWM (best $d$)
      & $96.50\pm1.04$ & $32.00$
      & $94.17\pm1.01$ & $256.0$
      & $85.83\pm0.44$ & $64.00$
      & $70.50\pm0.58$ & $16.00$ \\
                    & LeWM ($d=384$)
      & $83.17\pm0.17$ & $384.0$
      & $92.83\pm0.33$ & $384.0$
      & $84.67\pm1.20$ & $384.0$
      & $69.33\pm1.20$ & $384.0$ \\
                    & \MN{}
      & $\mathbf{100.0\pm0.00}$ & $8.16$
      & $\mathbf{95.50\pm0.29}$ & $55.36$
      & $\mathbf{86.17\pm0.88}$ & $32.00$
      & $\mathbf{77.67\pm1.20}$ & $32.00$ \\
    \bottomrule
  \end{tabular}%
  }
  }
\end{table}

Across all evaluated configurations, \MN{} achieves the highest mean success rate on all four tasks. Compared with full-width LeWM, it improves mean success by approximately $1.5-17$ percentage points while reducing average planning capacity by approximately $67-98\%$. Tuning LeWM’s latent dimension narrows the performance gap, but \MN{} retains gains of approximately $0.3-7.2$ percentage points, with the largest improvement on OGBench-Cube, while using fewer planning dimensions in $11$ of the 12 comparisons. These results suggest that a small planning state is only part of the design: how predictive information is learned and organized also matters. A wide embedding gives the encoder room to represent the observations, while prediction of the full next embedding encourages useful information to concentrate in a short prefix. Although the main goal of this paper is to concentrate information in a compact representation of a wide embedding vector, we also report the selected capacities per experiment for \MN{} in Table~\ref{tab:main-planning-capacity-counts} in Appendix~\ref{sec:main-results-implementation}. The table shows that selected capacities tend to be similar within an experiment for all episodes although in some cases they can vary across episodes, most noticeably on Reacher. We believe that this is because the environments are highly controlled with low-variation between episodes. In Appendix~\ref{sec:mixed_training} we train \MN{} on mixed datasets and show that it learns different capacity distributions per dataset, albeit with performance degradation. In addition, in Appendix~\ref{sec:reacher_capacity} we analyze the capacity allocation in Reacher, showing that when the goal is further away from the starting position lower capacities are preferred. We also found that capacities can vary across seeds, but for each evaluated dataset and backbone, selections remain confined to neighboring supported values. This variation may arise from the coarse capacity grid and training stochasticity. %leading to non-consistent capacity selection. 

\subsection{Analysis \& Ablation Study}
\label{sec:analysis}

\textbf{Training with fixed capacity probabilities.}
In our study, the model learns the capacity allocation during training. A natural question is how a model trained with a fixed capacity distribution would perform. This approach is reminiscent of nested dropout \cite{pmlr-v32-rippel14} and Matryoshka learning \cite{kusupati2022matryoshka}. We refer to this approach as Nested Dropout. We compare the methods using ViT-Tiny with $d_{\max}=192$. Nested Dropout uses the same capacity support, optimization settings, and relevant hyper-parameters as \MN{}, but samples from a fixed polynomial capacity distribution during training. We inspect two variants of Nested Dropout, one with SIGReg, and one with MixSIGReg using $q_\psi = \pi_0$.  %applies both prediction loss and SIGReg only to the active prefix. 
After training, we evaluate planning separately at each capacity in $\mathcal K$ and report the mean success rate of the best configuration per capacity in Figure~\ref{fig:nested-dropout-comparison}. 
% In both datasets, there is no clear advantage to either variant of the Nested Dropout across all dimensions, yet the best results for both are with the MixSIGReg variant. On PushT, Nested Dropout success rate improves strongly with active dimension and peaks at $92.50\pm0.76\%$ using $96$ dimensions for the MixSIGReg variant; \MN{} reaches $96.00\pm0.00\%$ success rate while using a shorter prefix of $53.49 \pm 10.51$. In OGBench-Cube, Nested Dropout results remain between $66.00\%$ and $70.00\%$ across the capacity range for the SIGReg variant and attains the highest success rate of $72.83 \pm 1.74$ using $192$ dimensions, while \MN{} reaches $79.00\pm0.76\%$ with active prefix of only $16.00$. Overall, this comparison favors the complete \MN{} training design over fixed-prior prefix training at compact planning capacities. Furthermore, Appendix~\ref{sec:cap_dynamics} presents the dynamics of capacity allocation during training showing non-trivial distribution shift as training progresses indicating on learning better capacity allocation and adapting it throughout the training process.
Neither Nested Dropout variant consistently outperforms the other across all capacities, although MixSIGReg achieves the highest success rate on both datasets. On PushT, Nested Dropout peaks at $92.50\pm0.76\%$ using $96$ dimensions with MixSIGReg; \MN{} reaches $96.00\pm0.00\%$ while using a shorter prefix of $53.49\pm10.51$ dimensions on average. In OGBench-Cube, Nested Dropout results remain between $65.83\%$ and $70.00\%$ across the capacity range for the SIGReg variant, while MixSIGReg attains its highest success rate of $72.83\pm1.74\%$ using $192$ dimensions. In comparison, \MN{} reaches $79.00\pm0.76\%$ with an active prefix of only $16$ dimensions. Overall, this comparison favors the complete \MN{} training design over fixed-prior prefix training at compact planning capacities. Furthermore, Appendix~\ref{sec:cap_dynamics} presents the dynamics of capacity allocation during training showing non-trivial distribution shift as training progresses indicating on learning better capacity allocation and adapting it throughout the training process.

%The figure places the selected $d_{\max}=192$ \MN{} models at their mean planning capacities, $\E[K_{\mathrm{plan}}]=53.49\pm10.51$ on PushT and $16.00\pm0.00$ on OGBench-Cube.

%\textbf{Closed-Loop experiments.}

\textbf{\MN{} capacity selection criteria.}
To isolate the role of capacity at test time, we re-evaluate $d_{\max}=192$ \MN{} models, %trained with the selected configurations in Table~\ref{tab:main-control-results}, 
overriding the selector with each of the $8$ supported fixed prefix sizes. No retraining is performed: the model, planning procedure, and active-prefix goal cost in Eq.~\ref{eq:active-prefix-goal-cost} are held fixed, and only the number of coordinates propagated through the rollout (including the goal state) is varied between experiments. We highlight that, like for Nested Dropout, this strategy requires evaluating each supported capacity separately to pick the best one based on the validation set, while the selection of \MN{} uses the selector network to choose a capacity per episode.  Figure~\ref{fig:rollout-prefix-sweep} compares the success rate of \MN{} with those obtained at fixed prefix sizes, averaged over three seeds. From the figure, no fixed prefix is uniformly best.
Moreover, while for OGBench-Cube \MN{}'s selection criterion matches the best fixed-prefix mean of $79.00\%$ at $k=16$, for PushT and Reacher the selected capacity has more variation. \MN{} achieves the success rate $96.00\%$ in PushT and the success rate $85.83\%$ in Reacher, matching or exceeding the success rates of $95.83\%$ in $k=96$ and $84.17\%$ in $k=128$ for those datasets, respectively, with smaller average prefix sizes. %Overall, across these evaluated configurations, the learned selector matches or exceeds the best mean success found by the eight-pass fixed-prefix sweep, while selecting smaller capacities on average.

\begin{figure}[!t]
  \centering
  \begin{minipage}[t]{0.32\linewidth}
    \vspace{0pt}\centering
    \captionsetup{font=footnotesize}
    \includegraphics[width=1.0\linewidth]{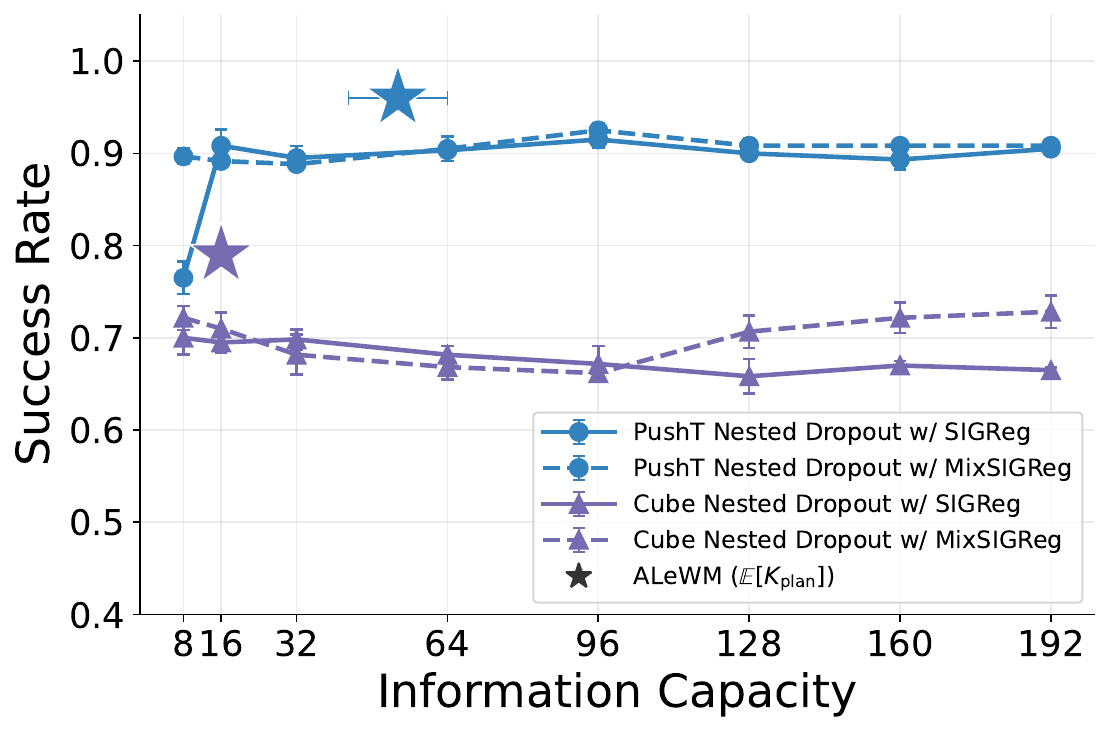}
    \captionof{figure}{Test success rate for Nested Dropout vs \MN{}.}
    \label{fig:nested-dropout-comparison}
  \end{minipage}\hfill
  \begin{minipage}[t]{0.32\linewidth}
    \vspace{0pt}\centering
    \captionsetup{font=footnotesize}
    \includegraphics[width=1.0\linewidth]{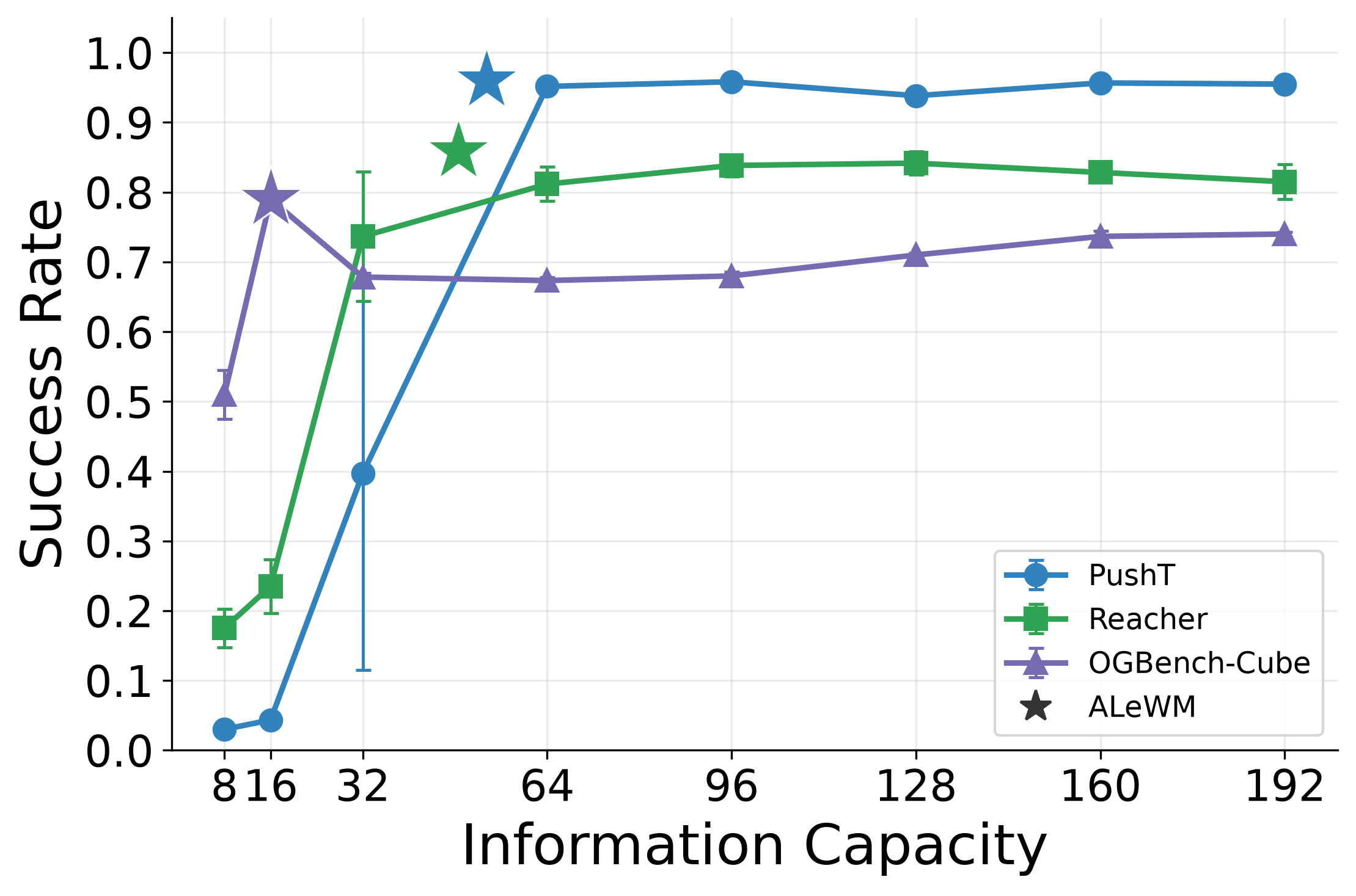}
    \captionof{figure}{Test success rate per fixed prefix size.}
    \label{fig:rollout-prefix-sweep}
  \end{minipage}\hfill
  \begin{minipage}[t]{0.32\linewidth}
    \vspace{0pt}\centering
    \captionsetup{font=footnotesize}
    \includegraphics[width=1.0\linewidth]{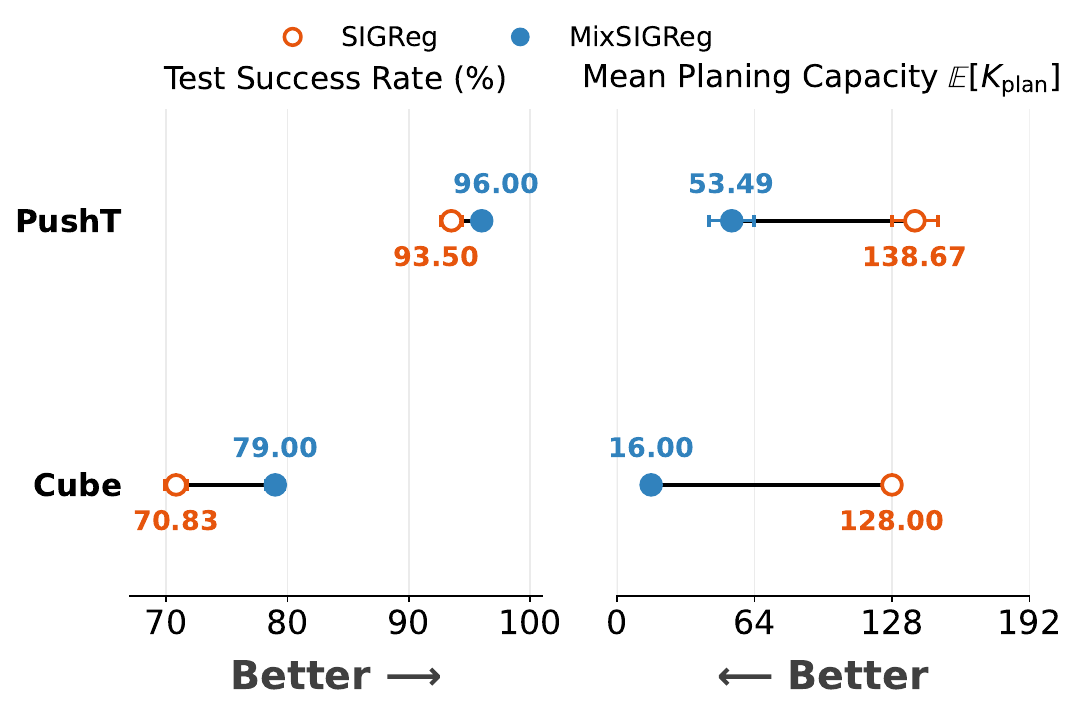}
    \captionof{figure}{\MN{} with MixSIGReg vs with SIGReg}
    \label{fig:sigreg-mixsigreg}
  \end{minipage}
\end{figure}

% TODO: Capacity-distribution diagnostics, mixed-dataset adaptation,
% prior-degree ablation, and example rollouts.
% \begin{table}[H]
%   \caption{MixSIGReg vs SIGReg on ViT-Tiny ($d_{\max}=192$) mean SR and capacity ($\pm$ SEM). }  %Results report the means $\pm$ SEM of test success rate (SR, \%) and selected test-time capacity $\E[K_{\text{plan}}]$ over three seeds.}
%   \label{tab:sigreg-vs-mixsigreg}
%   \centering
%   \captionsetup{font=small}
%   \scriptsize
%   \setlength{\tabcolsep}{3pt}
%   \scalebox{0.9}{%
%   \begin{tabular}{lrrrr}
%     \toprule
%     Regularizer
%       & \multicolumn{2}{c}{PushT}
%       & \multicolumn{2}{c}{OGBench-Cube} \\
%       & SR ($\uparrow$) & $\E[K_{\text{plan}}]$ ($\downarrow$)
%       & SR ($\uparrow$) & $\E[K_{\text{plan}}]$ ($\downarrow$) \\
%     \midrule
%     SIGReg
%       & $93.50\pm0.87$ & $138.67\pm10.67$
%       & $70.83\pm0.93$ & $128.00\pm0.00$ \\
%     MixSIGReg
%       & $\mathbf{96.00\pm0.00}$ & $\mathbf{53.49\pm10.51}$
%       & $\mathbf{79.00\pm0.76}$ & $\mathbf{16.00\pm0.00}$ \\
%     \bottomrule
%   \end{tabular}%
%   }
% \end{table}

\textbf{\MN{} with SIGReg.}
To verify the contribution of our proposed prior-weighted mixture regularizer (MixSIGReg) we compare \MN{} with MixSIGReg vs \MN{} with SIGReg on masked embeddings using the ViT-Tiny backbone and $d_{\max}=192$. For each regularizer and dataset, Figure~\ref{fig:sigreg-mixsigreg} reports the mean test success rate and capacity of the best configuration, with SEM over $3$ seeds. From the figure, MixSIGReg improves the mean test success rate by $2.50$ percentage points in PushT and $8.17$ points in OGBench-Cube, while reducing the mean selected prefix length from $138.67$ to $53.49$ and from $128.00$ to $16.00$, respectively. These results corroborate the importance of MixSIGReg both in terms of accuracy and capacity allocation.

\section{Conclusion}
This paper presents \MN{}, a method for adaptively selecting the active prefix size in world models. \MN{} introduces two components: a selector network that learns a probability distribution over active prefixes, and MixSIGReg, a regularizer for embeddings with varying prefix sizes that encourages a decreasing variance profile. Together, they encourage predictive information to concentrate in early coordinates through a self-supervised objective. Planning uses only the active prefixes using the largest selector probability. Across the evaluated benchmarks, \MN{} achieves higher mean success with compact planning prefixes, with the clearest gains on TwoRoom and OGBench-Cube. Several directions remain for improving the method. First, the capacity network is learned jointly with the world model, so its behavior can vary with initialization and other sources of randomness. Second, the MixSIGReg prior depends on a polynomial over capacity ranks and its degree. Alternatives that account for capacity values or avoid strict variance drops at capacity boundaries could improve the learned world model and potentially advance identifiability. Third, \MN{} uses the straight-through Gumbel--Softmax estimator, whose gradients are biased despite its practical effectiveness; exploring unbiased alternatives is a promising direction.

\bibliographystyle{iclr2027_conference}
\bibliography{references}
\newpage

\appendix
\section{LeWM Background}
\label{sec:background-details}

\textbf{Joint-embedding prediction.}
Given observations $\rvo_t,\rvo_{t+1}$ and action $\rva_t$, LeWM \citep{maes2026leworldmodel} forms:
\begin{equation}
  \rvs_t=\enc(\rvo_t),\qquad \rvs_{t+1}=\enc(\rvo_{t+1}),\qquad
  \hat \rvs_{t+1}=\pred(\rvs_t, \rva_t).
  \label{eq:summary-lewm-map}
\end{equation}
Up to empirical averaging and scalar loss weights, the LeWM objective is:
\begin{equation}
  \mathcal L_{\mathrm{LeWM}}
  =
  \left\|\hat \rvs_{t+1}-\rvs_{t+1}\right\|_2^2
  +
  \lambda\,\SIGReg(\rmS),
  \label{eq:summary-lewm-loss}
\end{equation}
where $\rmS$ is the batch/history tensor of full embeddings.  In practice,
LeWM uses a window of preceding embeddings and actions rather than only the
latest pair as shown in Eq.~\ref{eq:summary-lewm-map}.

\textbf{SIGReg regularization.}
Let $\rmS\in\sR^{F\times B\times d}$ contain $F$ frame positions for each of
$B$ sequences.  For $P$ projection directions
$\rvu^{(p)}\in\sS^{d-1}$, SIGReg \citep{balestriero2025lejepa} computes the empirical characteristic
function separately at each frame position $t$,
\begin{equation}
    \hat\phi_t^{(p)}(\tau)
    =\frac{1}{B}\sum_{b=1}^B
    \exp\!\left(i\tau\rvu^{(p)\top}\rvs_{b,t}\right),
\end{equation}
compares it with the standard-normal characteristic function
$\phi_{\gN}(\tau)=\exp(-\tau^2/2)$ through the weighted Epps--Pulley statistic,
\begin{equation}
\begin{aligned}
        EP_t^{(p)}
    =B\int w(\tau)
    \left|\hat\phi_t^{(p)}(\tau)-\phi_{\gN}(\tau)\right|^2d\tau,
\end{aligned}
\label{eq:ep}
\end{equation}
averaged over $p$ and $t$. The batch-size factor $B$ is included in both SIGReg and MixSIGReg. Both use the same frequency window and quadrature, specified in Appendix~\ref{sec:main-results-implementation}.

\textbf{Latent planning.}
Planning in LeWM optimizes candidate action sequences by rolling out the (trained) predictor in latent space based on the encoding of the initial and goal observations $\rvo_1$ and $\rvo_g$, respectively:
\begin{equation}
    \hat\rvs_1=\enc(\rvo_1),
    \qquad
    \hat\rvs_{t+1}=\pred(\hat\rvs_t,\rva_t),
    \qquad
    \argmin_{\rva_{1:H-1}}
    \left\|\hat\rvs_H-\enc(\rvo_g)\right\|_2^2.
  \label{eq:latent-planing}
\end{equation}
Here $H$ is the planning horizon.  Candidate actions are optimized using the Cross-Entropy Method (CEM) \citep{rubinstein2004cross}.

\section{Straight-through prefix sampling}
\label{sec:st_estimator}
Define $q_c \coloneqq q_\psi(k_c | \rvw)$ an element in the vector of probabilities under $\rvq \in \Delta^{C-1}$. In addition, define a matrix of masks $\rmM \in \{0,1\}^{C \times d}$, with row $c$ equal to $\rvm(k_c)^\top$.
To sample a prefix mask we apply the following steps:
\begin{enumerate}
    \item Draw $C$ independent Gumbel variables
    $g_c=-\log[-\log(U_c)]$, with
    $U_c\sim\operatorname{Uniform}(0,1)$, and collect them in
    $\rvg=(g_1,\ldots,g_C)$.
    \item Produce a Gumbel-Softmax sample:
    \begin{equation}
      \boldsymbol{\chi}_{\mathrm{soft}}
      =\operatorname{softmax}\!\left(
      \frac{\log\rvq+\rvg}{\eta}\right),
      \qquad
      c^\ast
      =\argmax_c\chi_{\mathrm{soft}}^c
      =\argmax_c(\log q_c+g_c).
      \label{eq:gs_soft_probs}
    \end{equation}
    where $\eta>0$ is the temperature set to $0.5$ throughout.
    \item Construct
    $\boldsymbol{\chi}_{\mathrm{hard}}=\operatorname{onehot}(c^\ast)
    \in\{0,1\}^C$.
    \item Use the straight-through vector
    $\boldsymbol{\chi}=\boldsymbol{\chi}_{\mathrm{hard}}
    -\operatorname{stopgrad}(\boldsymbol{\chi}_{\mathrm{soft}})
    +\boldsymbol{\chi}_{\mathrm{soft}}$.
    \item Form the prefix mask $\rvm=\rmM^\top\boldsymbol{\chi}$.
\end{enumerate}

\section{Theoretical Claims}
\label{sec:theory-details}

We next formalize the two mechanisms induced by the capacity-matched target and nested prediction.  %The first mechanism prescribes an ordered variance budget for the masked aggregate representation.  The second makes earlier coordinates available to the predictor no less often than later coordinates. 
We distinguish the corresponding prior and selector survival probabilities throughout.

Let $K_\pi\sim\pi_0(\cdot;\alpha)$ and let $\boldsymbol{\epsilon}_0\sim\gN(0,\rmI_d)$ be independent.  The random variable $\rvz_0=\rmM_{K_\pi}\boldsymbol{\epsilon}_0$ has distribution $p_0$ from
Eq.~\ref{eq:prior-marginal-target}.  Define the prior survival probability
\begin{equation}
    \rho_\pi^j
    =
    \Pr(K_\pi\geq j)
    =
    \sum_{\substack{k\in\mathcal K\\k\geq j}}\pi_0(k;\alpha),
    \qquad j=1,\ldots,d.
    \label{eq:prior-survival-probability}
\end{equation}

\begin{proposition}[Ordered aggregate variance]
\label{prop:ordered-aggregate-variance}
The capacity-matched target satisfies
\begin{equation}
    \E[\rvz_0]=0,
    \qquad
    \operatorname{Cov}(\rvz_0)
    =
    \diag(\rho_\pi^1,\ldots,\rho_\pi^d),
    \qquad
    \rho_\pi^1\geq\cdots\geq\rho_\pi^d.
    \label{eq:ordered-target-covariance}
\end{equation}
Moreover,
\begin{equation}
    \Tra\!\left(\operatorname{Cov}(\rvz_0)\right)
    =
    \sum_{j=1}^d\rho_\pi^j
    =
    \E[K_\pi],
    \qquad
    \rho_\pi^j-\rho_\pi^{j+1}=\pi_0(j;\alpha),
    \label{eq:ordered-target-trace}
\end{equation}
where $\rho_\pi^{d+1}=0$ and $\pi_0(j;\alpha)=0$ when
$j\notin\mathcal K$.  Hence, the target variance is non-increasing and drops
only at supported capacity boundaries; it is constant within every intervening
coordinate block.
\end{proposition}

\begin{proof}
Conditional on $K_\pi=k$, the target has mean zero and covariance $\rmM_k$.
The law of total covariance therefore gives
\begin{equation}
    \operatorname{Cov}(\rvz_0)
    =
    \sum_{k\in\mathcal K}\pi_0(k;\alpha)\rmM_k.
\end{equation}
The $j$th diagonal entry equals
$\sum_{k\geq j}\pi_0(k;\alpha)=\rho_\pi^j$, while every off-diagonal entry is
zero.  The events $\{K_\pi\geq j+1\}\subseteq\{K_\pi\geq j\}$ imply
monotonicity, and their probability difference is
$\Pr(K_\pi=j)=\pi_0(j;\alpha)$.  Finally, exchanging the order of two finite
sums yields
\begin{equation}
    \sum_{j=1}^d\rho_\pi^j
    =
    \sum_{k\in\mathcal K}\pi_0(k;\alpha)
    \sum_{j=1}^d\mathds{1}[j\leq k]
    =
    \sum_{k\in\mathcal K}k\pi_0(k;\alpha)
    =
    \E[K_\pi].
\end{equation}
\end{proof}

Proposition~\ref{prop:ordered-aggregate-variance} says that the capacity prior sets a variance budget for the masked reference distribution.  Earlier coordinates receive at least as much of this budget because every cutoff that retains a later coordinate also retains the earlier ones, and the total budget equals the prior's expected capacity.  Coordinates in the same capacity block receive the same budget, which drops only where the capacity support permits a cutoff.  This statement concerns only the reference distribution; however by optimizing MixSIGReg the learned selector distribution is encouraged to maintain that structure.

We next separate this prior target from the distribution induced by the learned selector.  Let the context $\rvw$ follow the data distribution, draw $K\mid\rvw\sim q_\psi(\cdot\mid\rvw)$, and let $P_{\theta,\psi,t}$ denote the
resulting population distribution of $\rmM_K\rvs_t$.  Unless stated otherwise, expectations involving $K$, $\rvw$, and $\rvs_t$ are under this joint law. Define the conditional and aggregate selector survival probabilities:
\begin{equation}
    \rho_q^j(\rvw)
    =
    \Pr_{q_\psi}(K\geq j\mid\rvw)
    =
    \sum_{\substack{k\in\mathcal K\\k\geq j}}q_\psi(k\mid\rvw),
    \qquad
    \bar\rho_q^j=\E_{\rvw}[\rho_q^j(\rvw)].
    \label{eq:selector-survival-probability}
\end{equation}

To state what exact distribution matching transfers to the learned masked latent, let $\omega$ be a finite Borel measure whose support is all of $\sR^d$, and define the idealized full-frequency characteristic-function discrepancy.  Here $\phi_{\theta,\psi,t}$ and $\phi_0$ denote the characteristic
functions of $P_{\theta,\psi,t}$ and $p_0$, respectively:
\begin{equation}
    D_\omega^2(P_{\theta,\psi,t},p_0)
    =
    \int_{\sR^d}
    \left|\phi_{\theta,\psi,t}(\boldsymbol{\xi})-\phi_0(\boldsymbol{\xi})\right|^2
    d\omega(\boldsymbol{\xi}).
    \label{eq:population-characteristic-discrepancy}
\end{equation}

\begin{proposition}[Population target-law matching]
\label{prop:population-target-matching}
For $\omega$ as defined above,
\begin{equation}
    D_\omega^2(P_{\theta,\psi,t},p_0)=0
    \quad\Longleftrightarrow\quad
    P_{\theta,\psi,t}=p_0.
    \label{eq:population-target-equivalence}
\end{equation}
\end{proposition}

\begin{proof}
The reverse implication is immediate.  For the forward implication,
$D_\omega^2=0$ gives
$\phi_{\theta,\psi,t}(\boldsymbol{\xi})=\phi_0(\boldsymbol{\xi})$ for
$\omega$-almost every $\boldsymbol{\xi}$.  Characteristic functions are
continuous.  If they differed at one
point, continuity would give a nonempty open neighborhood on which their
difference remained nonzero, contradicting the full support of $\omega$.
They therefore agree on all of $\sR^d$, and uniqueness of characteristic
functions gives $P_{\theta,\psi,t}=p_0$.
\end{proof}

\begin{corollary}[Learned masked-variance profile]
\label{cor:learned-masked-variance}
Suppose $D_\omega^2(P_{\theta,\psi,t},p_0)=0$.  Define the sequence-level
activation indicator $m_K^j\coloneqq[\rvm(K)]^j=\mathds{1}[K\geq j]$, so that
$z_t^j=m_K^j s_t^j$.  Then, for every coordinate $j$,
\begin{equation}
    \E[m_K^j s_t^j]=0,
    \qquad
    \Var(z_t^j)
    =\E[m_K^j(s_t^j)^2]
    =\rho_\pi^j.
    \label{eq:learned-masked-coordinate-variance}
\end{equation}
Consequently,
\begin{equation}
    \Var(z_t^1)\geq\cdots\geq\Var(z_t^d),
    \qquad
    \sum_{j=1}^d\Var(z_t^j)=\E[K_\pi].
    \label{eq:learned-masked-variance}
\end{equation}
Thus, the learned \emph{masked aggregate} has the non-increasing target
variance profile.  Moreover, if $\bar\rho_q^j>0$, then
\begin{equation}
    \E[s_t^j\mid m_K^j=1]=0,
    \qquad
    \Var(s_t^j\mid m_K^j=1)
    =\frac{\rho_\pi^j}{\bar\rho_q^j}.
    \label{eq:active-coordinate-variance}
\end{equation}
If $\rho_\pi^j>0$, exact matching necessarily implies
$\bar\rho_q^j>0$.  To relate this result to the unmasked encoder coordinate,
suppose additionally that $\E[(s_t^j)^2]<\infty$.  When
$\bar\rho_q^j<1$, define
\begin{equation}
    \mu_0^j=\E[s_t^j\mid m_K^j=0],
    \qquad
    v_0^j=\Var(s_t^j\mid m_K^j=0).
    \label{eq:inactive-coordinate-moments}
\end{equation}
Then the three possible activation regimes give
\begin{equation}
\begin{aligned}
    \Var(s_t^j)
    &=
    \begin{cases}
        v_0^j,
        & \bar\rho_q^j=0 \quad (\rho_\pi^j=0),\\[1mm]
        \rho_\pi^j+(1-\bar\rho_q^j)v_0^j
        +\bar\rho_q^j(1-\bar\rho_q^j)(\mu_0^j)^2,
        & 0<\bar\rho_q^j<1,\\[1mm]
        \rho_\pi^j,
        & \bar\rho_q^j=1,
    \end{cases}\\[-1mm]
    &\geq \rho_\pi^j.
\end{aligned}
    \label{eq:unmasked-coordinate-variance}
\end{equation}
\end{corollary}

\begin{proof}
Proposition~\ref{prop:population-target-matching} gives
$P_{\theta,\psi,t}=p_0$, so Proposition~
\ref{prop:ordered-aggregate-variance} gives
$\E[z_t^j]=0$, $\E[(z_t^j)^2]=\rho_\pi^j$, and the stated total variance.
Since $z_t^j=m_K^j s_t^j$ and $(m_K^j)^2=m_K^j$, these are exactly the first
two moment identities in Eq.~\ref{eq:learned-masked-coordinate-variance}.

All learned-side expectations below are under the joint law specified above.
By iterated expectation,
\begin{align*}
    \Pr(m_K^j=1)
    &=\E_{\rvw}\!\left[
        \E_{K\mid\rvw}[\mathds{1}[K\geq j]]
      \right] \\
    &=\E_{\rvw}\!\left[
        \sum_{k\in\mathcal K}
        q_\psi(k\mid\rvw)\mathds{1}[k\geq j]
      \right] \\
    &=\E_{\rvw}[\rho_q^j(\rvw)]
     =\bar\rho_q^j.
\end{align*}
For $r\in\{1,2\}$ and $\bar\rho_q^j>0$, the defining indicator identity for
conditional expectation, followed by iterated expectation, gives the full
calculation
\begin{align*}
    \E[(s_t^j)^r\mid m_K^j=1]
    &=\frac{
        \E[\mathds{1}[m_K^j=1](s_t^j)^r]
      }{
        \Pr(m_K^j=1)
      } \\
    &=\frac{1}{\bar\rho_q^j}
      \E_{\rvw}\!\left[
        \E_{K\mid\rvw}\!\left[
          \mathds{1}[K\geq j](s_t^j)^r
        \right]
      \right] \\
    &=\frac{1}{\bar\rho_q^j}
      \E_{\rvw}\!\left[
        \sum_{k\in\mathcal K}
        q_\psi(k\mid\rvw)\mathds{1}[k\geq j](s_t^j)^r
      \right] \\
    &=\frac{1}{\bar\rho_q^j}
      \E_{\rvw}[\rho_q^j(\rvw)(s_t^j)^r] \\
    &=\frac{\E[m_K^j(s_t^j)^r]}{\bar\rho_q^j}.
\end{align*}
The third equality uses that $s_t^j$ is fixed once $\rvw$ is fixed.  This is
the conditional-expectation form of Bayes' rule; under a density formulation,
the factor $\rho_q^j(\rvw)/\bar\rho_q^j$ is precisely the likelihood-ratio
weight that changes the data law to the law conditioned on $m_K^j=1$.
Taking $r=1$ and $r=2$ and using the first two moment identities gives
\begin{equation*}
    \E[s_t^j\mid m_K^j=1]=0,
    \qquad
    \E[(s_t^j)^2\mid m_K^j=1]
    =\frac{\rho_\pi^j}{\bar\rho_q^j}.
\end{equation*}
The conditional mean is zero, so the conditional second moment equals the
conditional variance, proving Eq.~\ref{eq:active-coordinate-variance}.
Finally, if $\bar\rho_q^j=0$, then $z_t^j=0$ almost surely and hence has zero
variance; this contradicts $\rho_\pi^j>0$ under exact matching.
For $0<\bar\rho_q^j<1$, the law of total variance conditional on $m_K^j$ gives
\begin{align*}
    \Var(s_t^j)
    ={}&\E\!\left[\Var(s_t^j\mid m_K^j)\right]
    +\Var\!\left(\E[s_t^j\mid m_K^j]\right),\\
    \E\!\left[\Var(s_t^j\mid m_K^j)\right]
    ={}&\bar\rho_q^j\Var(s_t^j\mid m_K^j=1)
    +(1-\bar\rho_q^j)v_0^j.
\end{align*}
For the second term, the active conditional mean is zero and the inactive
conditional mean is $\mu_0^j$, so
\begin{align*}
    \Var\!\left(\E[s_t^j\mid m_K^j]\right)
    &={\E\!\left[\bigl(\E[s_t^j\mid m_K^j]\bigr)^2\right]}
      -{\E\!\left[\E[s_t^j\mid m_K^j]\right]^2}\\
    &=\bar\rho_q^j 0^2+(1-\bar\rho_q^j)(\mu_0^j)^2\\
    &\quad-
      \bigl(\bar\rho_q^j 0+(1-\bar\rho_q^j)\mu_0^j\bigr)^2\\
    &=\bar\rho_q^j(1-\bar\rho_q^j)(\mu_0^j)^2.
\end{align*}
Combining these terms and substituting
Eq.~\ref{eq:active-coordinate-variance} gives the middle case of
Eq.~\ref{eq:unmasked-coordinate-variance}.  If $\bar\rho_q^j=0$, then
$m_K^j=0$ almost surely, so $\Var(s_t^j)=v_0^j$, while exact matching gives
$\rho_\pi^j=\Var(z_t^j)=0$.  If $\bar\rho_q^j=1$, then $m_K^j=1$ almost surely,
so $s_t^j=z_t^j$ and $\Var(s_t^j)=\rho_\pi^j$.  These observations also prove
the lower bound in Eq.~\ref{eq:unmasked-coordinate-variance}.
\end{proof}

\paragraph{Intuition.}
Equation~\ref{eq:active-coordinate-variance} shows how activation frequency
and active-coordinate scale can compensate for one another: for fixed target
variance $\rho_\pi^j$, increasing $\bar\rho_q^j$ decreases the required
conditional variance when the coordinate is active, whereas decreasing
$\bar\rho_q^j$ increases it.  Thus the equation fixes only the product
$\bar\rho_q^j\Var(s_t^j\mid m_K^j=1)=\rho_\pi^j$, not either factor
separately.  Equations~\ref{eq:inactive-coordinate-moments}--
\ref{eq:unmasked-coordinate-variance} also include the encoder values produced
while the coordinate is inactive.  Variation among those values, or a shift
between their mean and the active mean, adds the two nonnegative terms in the
middle case.  These additions can differ across coordinates.  In plain words,
the masked coordinates $z_t^j$ have an ordered variance profile, but the raw
coordinates $s_t^j$ need not inherit that order: a later raw coordinate may
have larger variance than an earlier one.  The raw variances may still happen
to be ordered in a learned representation; the target-law equality simply does
not guarantee it.  This freedom concerns the masked target-law constraint
alone, and the shared model and prediction loss may still constrain inactive
values.  The identities also require exact population matching; a small finite
MixSIGReg surrogate need not satisfy them.

We next quantify the predictive value of nested prefixes while holding the
representation fixed.  Fix the encoder and population data distribution, and
consider a fixed prediction time $t$.  Write $\E_{\mathrm{data}}$ for
expectation under the induced joint law of
$(\rvs_{1:t+1},\rva_{1:t})$, and assume
$\E_{\mathrm{data}}\|\rvs_{t+1}\|_2^2<\infty$.  Write
$\rvs_{1:t}^{1:k}=\{s_\tau^j:1\leq\tau\leq t,\ 1\leq j\leq k\}$ for the first
$k$ coordinates of the latent history, and define
\begin{equation}
    \mathcal F_k
    =
    \sigma(\rvs_{1:t}^{1:k},\rva_{1:t}),
    \qquad
    \rvmu_k=\E_{\mathrm{data}}[\rvs_{t+1}\mid\mathcal F_k],
    \qquad
    R_k=\E_{\mathrm{data}}\|\rvs_{t+1}-\rvmu_k\|_2^2,
    \label{eq:prefix-bayes-risk}
\end{equation}
where $\mathcal F_0=\sigma(\rva_{1:t})$.  Here $\sigma(\cdot)$ denotes the
sigma-algebra generated by its arguments: the collection of events whose
occurrence can be determined from those random variables.  Thus,
$\mathcal F_k$ contains all information in the action history and the first
$k$ coordinates of the latent history, while $\mathcal F_0$ contains only the
action history.  The conditional mean $\rvmu_k$ is the Bayes-optimal
squared-error prediction using that information.  Accordingly, $R_k$ is a
\emph{forced-capacity oracle risk}: it is the population squared error of a
separate Bayes-optimal predictor at prefix $k$.  It does not assume that the
single finite predictor used in training attains these risks. We note that because the
training selector can observe target frames, its mask can convey information
absent from $\mathcal F_k$. Hence, these capacity risks are not automatically lower bounds for a target-dependent selection mechanism, but serve as a useful surrogate.

\begin{proposition}[Survival-weighted predictive value]
\label{prop:survival-weighted-predictive-value}
For $k\leq\ell$,
\begin{equation}
    R_k-R_\ell
    =
    \E_{\mathrm{data}}\|\rvmu_\ell-\rvmu_k\|_2^2
    \geq0.
    \label{eq:nested-bayes-risk-difference}
\end{equation}
Set $k_0=0$ and define the predictive gain of capacity block $c$ as
\begin{equation}
    \Delta_c=R_{k_{c-1}}-R_{k_c}\geq0.
\end{equation}
If $K_\pi\sim\pi_0(\cdot;\alpha)$ is sampled independently of the data, then
\begin{equation}
    \E_{K_\pi}[R_{K_\pi}]
    =
    R_0-
    \sum_{c=1}^C\rho_\pi^{k_c}\Delta_c.
    \label{eq:survival-weighted-risk-decomposition}
\end{equation}
Here $\E_{K_\pi}$ averages only over the independent prior-capacity draw.
%each $R_k$ has already averaged over the population data law.  Independence
%therefore lets the left-hand side also represent the joint population risk of
%sampling a prior capacity and then predicting a data draw at that capacity.
Thus, the predictive gain of each block is weighted by the probability that the block survives the sampled capacity. If, additionally, every block permutation $\sigma$ is feasible
and yields prefix risks
$R_{k_r}^{(\sigma)}=R_0-\sum_{c=1}^{r}\Delta_{\sigma(c)}$
under the same capacity prior, then
$\E_{K_\pi}[R_{K_\pi}^{(\sigma)}]$ is minimized by arranging
the gains in non-increasing order. %This prior-weighted expression is a data-independent reference case.  
For the learned selector, let
$\ell_{k_c}(\rvw)$ denote any forced-capacity loss at context $\rvw$, let
$\ell_0(\rvw)$ denote the corresponding all-zero-prefix loss, and define
$\delta_c(\rvw)=\ell_{k_{c-1}}(\rvw)-\ell_{k_c}(\rvw)$.  Then, context by
context,
\begin{equation}
    \E_{K\sim q_\psi(\cdot\mid\rvw)}[\ell_K(\rvw)]
    =
    \ell_0(\rvw)-
    \sum_{c=1}^C\rho_q^{k_c}(\rvw)\delta_c(\rvw).
    \label{eq:conditional-selector-risk-decomposition}
\end{equation}
The expectation in Eq.~\ref{eq:conditional-selector-risk-decomposition} is only
over the categorical draw of $K$, with the context $\rvw$ held fixed.
The increments $\delta_c(\rvw)$ need not be positive for a shared finite
predictor.
\end{proposition}

\begin{proof}
For $k\leq\ell$, $\mathcal F_k\subseteq\mathcal F_\ell$.  The space
$L^2(\mathcal F_k;\sR^d)$ is the closed subspace of square-integrable,
$\mathcal F_k$-measurable random vectors, with inner product
$\langle\rvu,\rvv\rangle_{L^2}
=\E_{\mathrm{data}}[\rvu^\top\rvv]$.  The conditional mean $\rvmu_k$ is the
orthogonal projection of $\rvs_{t+1}$ onto this subspace.  The projections are
\emph{nested} because the larger prefix gives the inclusion
$L^2(\mathcal F_k;\sR^d)\subseteq L^2(\mathcal F_\ell;\sR^d)$.  In particular,
$\rvmu_\ell-\rvmu_k$ is $\mathcal F_\ell$-measurable, whereas the residual
$\rvs_{t+1}-\rvmu_\ell$ is orthogonal to every vector in that larger subspace.
Explicitly,
\begin{align*}
    &\E_{\mathrm{data}}\!\left[
      (\rvs_{t+1}-\rvmu_\ell)^\top(\rvmu_\ell-\rvmu_k)\right]\\
    &\quad=\E_{\mathrm{data}}\!\left[
      \E_{\mathrm{data}}\!\left[
      (\rvs_{t+1}-\rvmu_\ell)^\top(\rvmu_\ell-\rvmu_k)
      \mid\mathcal F_\ell\right]\right]\\
    &\quad=\E_{\mathrm{data}}\!\left[
      \E_{\mathrm{data}}[\rvs_{t+1}-\rvmu_\ell\mid\mathcal F_\ell]^\top
      (\rvmu_\ell-\rvmu_k)\right]
    =0.
\end{align*}
Using
$\rvs_{t+1}-\rvmu_k=(\rvs_{t+1}-\rvmu_\ell)
+(\rvmu_\ell-\rvmu_k)$ and expanding the squared norm now gives
\begin{align*}
    R_k
    &=\E_{\mathrm{data}}\|\rvs_{t+1}-\rvmu_\ell\|_2^2
      +\E_{\mathrm{data}}\|\rvmu_\ell-\rvmu_k\|_2^2\\
    &\quad+2\E_{\mathrm{data}}\!\left[
      (\rvs_{t+1}-\rvmu_\ell)^\top(\rvmu_\ell-\rvmu_k)\right]\\
    &=R_\ell+\E_{\mathrm{data}}\|\rvmu_\ell-\rvmu_k\|_2^2,
\end{align*}
which proves Eq.~\ref{eq:nested-bayes-risk-difference}.

For a supported capacity $k_r$, telescoping gives
$R_{k_r}=R_0-\sum_{c=1}^r\Delta_c$.  Since $K_\pi$ has the finite support
$\{k_1,\ldots,k_C\}$, taking expectation over the prior-capacity draw and
exchanging the two finite sums gives the full calculation
\begin{align}
    \E_{K_\pi}[R_{K_\pi}]
    &=\sum_{r=1}^C\pi_0(k_r;\alpha)R_{k_r}\notag\\
    &=\sum_{r=1}^C\pi_0(k_r;\alpha)
      \left(R_0-\sum_{c=1}^r\Delta_c\right)\notag\\
    &=R_0-\sum_{c=1}^C
      \left(\sum_{r=c}^C\pi_0(k_r;\alpha)\right)\Delta_c\notag\\
    &=R_0-\sum_{c=1}^C\Pr(K_\pi\geq k_c)\Delta_c.
    \label{eq:prior-risk-sum-exchange}
\end{align}
The last equality uses
$\{K_\pi\geq k_c\}=\{K_\pi\in\{k_c,\ldots,k_C\}\}$.  Since
$\Pr(K_\pi\geq k_c)=\rho_\pi^{k_c}$,
Eq.~\ref{eq:prior-risk-sum-exchange} is exactly
Eq.~\ref{eq:survival-weighted-risk-decomposition}. For $i<j$ with $\Delta_i<\Delta_j$, exchanging the two gains
reduces the expected risk by
$(\rho_\pi^{k_i}-\rho_\pi^{k_j})(\Delta_j-\Delta_i)\geq0$;
successive exchanges establish the ordering claim.
For the final identity, the same telescope gives
$\ell_{k_r}(\rvw)=\ell_0(\rvw)-\sum_{c=1}^r\delta_c(\rvw)$.  Taking expectation
over $K\sim q_\psi(\cdot\mid\rvw)$ and exchanging the two finite sums yields
Eq.~\ref{eq:conditional-selector-risk-decomposition}.
\end{proof}

\paragraph{Predictive-allocation interpretation.}
Proposition~\ref{prop:survival-weighted-predictive-value} explains why a
component's position in the prefix matters.  In the independent-prior case of
Eq.~\ref{eq:survival-weighted-risk-decomposition}, suppose each component has a
fixed nonnegative gain, and these gains add without changing when components
are moved or combined.  Each position inherits the probability that its block
survives the cutoff, so Proposition~\ref{prop:survival-weighted-predictive-value}
weights each component's gain by that probability.  A component placed earlier
is retained at least as often; therefore, exchanging a larger late gain with a
smaller early gain cannot increase expected prediction error, and strictly
decreases it when their retention probabilities differ.  Positions with the
same retention probability remain unordered.  For the learned selector,
Eq.~\ref{eq:conditional-selector-risk-decomposition} provides the same
accounting within each context, but the gains can interact, change when
components move, or become negative, and one shared representation need not
realize every context-specific ordering.  This argument therefore explains the
ordering pressure created by Proposition~\ref{prop:survival-weighted-predictive-value};
it is not a convergence guarantee for joint training.

% \paragraph{Interpretation and scope.}
% Proposition~\ref{prop:ordered-aggregate-variance} is an exact property of the target, while the non-increasing exposure probabilities follow directly from the hard nested masks.  Proposition~\ref{prop:population-target-matching} and Corollary~
% \ref{cor:learned-masked-variance} transfer the target variance profile only under the stronger population, full-frequency matching premise. Proposition~\ref{prop:survival-weighted-predictive-value} gives a fixed-representation
% oracle benchmark and the exact contextual telescoping identity.  Under the additional additive, context-independent assumptions above, its survival weights favor assigning larger predictive gains to earlier blocks. Note that these one-step prediction results do not guarantee accurate recursive prefix rollouts or preservation of goal-relevant information.

\section{Additional Experiments \& Analysis}
\subsection{Rollout target: active vs full} 
\label{sec:rollout_target_act_vs_full}
Here, we test whether planning should compare the predicted (unmasked) state $\rvs_{H}$ with the full goal embedding ($\rvs_g$) or only with the selected prefix embedding by the capacity network according to Eq.~\ref{eq:active-prefix-goal-cost}. Both conditions keep the selected prefix fixed and mask all intermediate rollout states. For each setting, we re-evaluate ViT-Tiny ($d_{\max}=192$) checkpoints without retraining.  Table~\ref{tab:rollout-active-vs-full} shows that restricting the rollout objective to active dimensions improves the success rate for both OGBench-Cube and Reacher, with an unchanged mean success on PushT. This experiment supports our selection, providing an indication that the information important for accurate prediction of the goal state concentrates mostly at the prefix.

\begin{table}[H]
  \caption{Effect of the rollout target capacity size on test success rate (SR, \%) with ViT-Tiny \MN{} having $d_{\max}=192$.  Results present test success rate (SR, \%) $\pm$ SEM over 3 seeds.}
  \label{tab:rollout-active-vs-full}
  \centering
  \small
  \setlength{\tabcolsep}{4pt}
  \begin{tabular}{lrrr}
    \toprule
    Rollout target
      & PushT SR ($\uparrow$)
      & Reacher SR ($\uparrow$)
      & OGBench-Cube SR ($\uparrow$) \\
    \midrule
    Full Capacity $(d_{max})$
      & $\mathbf{96.00\pm0.58}$
      & $83.50\pm2.25$
      & $75.17\pm0.44$ \\
    Max Probability $(K_{\text{plan}})$
      & $\mathbf{96.00\pm0.00}$
      & $\mathbf{85.83\pm1.01}$
      & $\mathbf{79.00\pm0.76}$ \\
    \bottomrule
  \end{tabular}
\end{table}

\subsection{Capacity dynamics}
\label{sec:cap_dynamics}
Figure~\ref{fig:capacity-dynamics} tracks the learned categorical capacity distribution during training for one ViT-Tiny run of \MN{} with $d_{\max}=192$ for PushT, Reacher, and OGBench-Cube. The figure shows that in all three datasets, most redistribution occurs during the first few thousand optimization steps, after which the probabilities evolve more gradually, sometimes crossing one another. The converged mixtures are task dependent: PushT and Reacher assign the largest mean probability to $k=32$, whereas OGBench-Cube assigns the largest probability to $k=16$.  OGBench-Cube also concentrates more strongly on the two smallest nonzero-probability capacities, $k=16$ and $k=32$. By contrast, PushT and Reacher retain more mass across $k\in\{32,64,96\}$.  The probability of $k=8$ approaches zero in every dataset.  These trajectories show that the capacity selector does not merely preserve its initialization: it learns distinct distributions for the three control domains.

\begin{figure*}[t]
  \centering
  \begin{subfigure}[t]{0.32\linewidth}
    \centering
    \includegraphics[width=\linewidth]{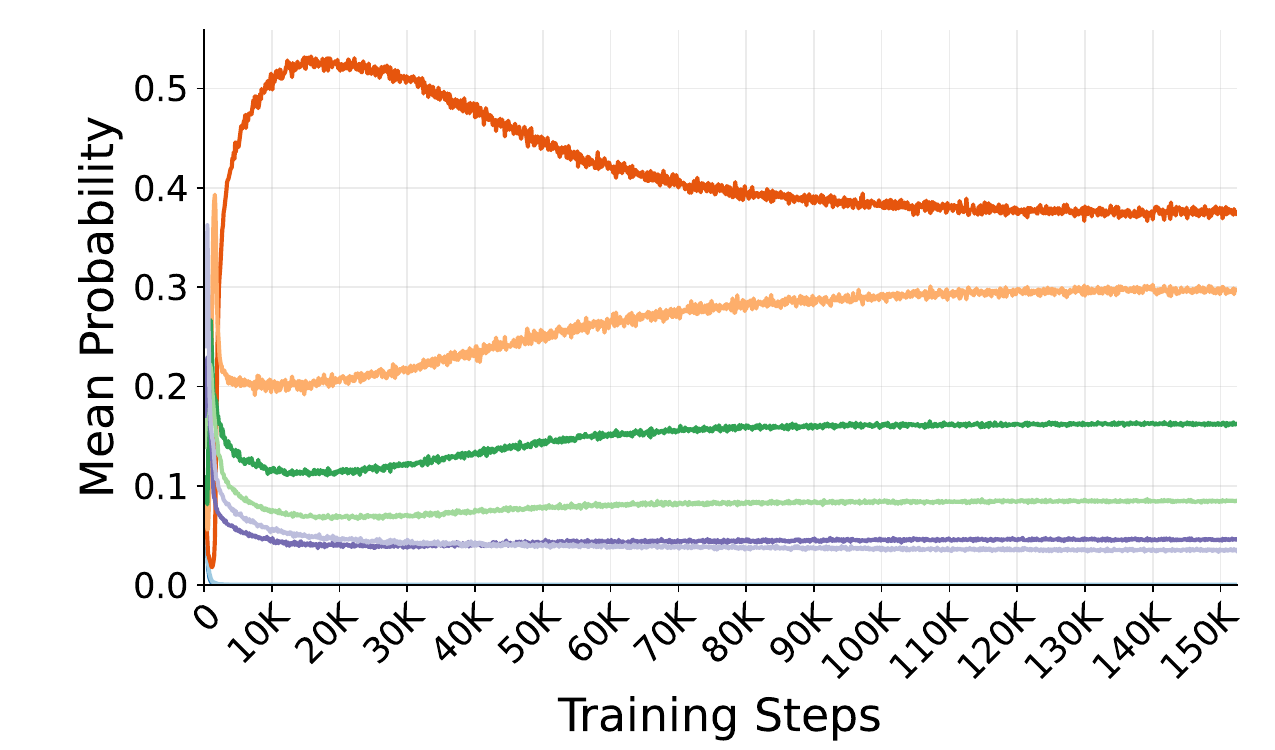}
    \caption{PushT.}
  \end{subfigure}
  \hfill
  \begin{subfigure}[t]{0.32\linewidth}
    \centering
    \includegraphics[width=\linewidth]{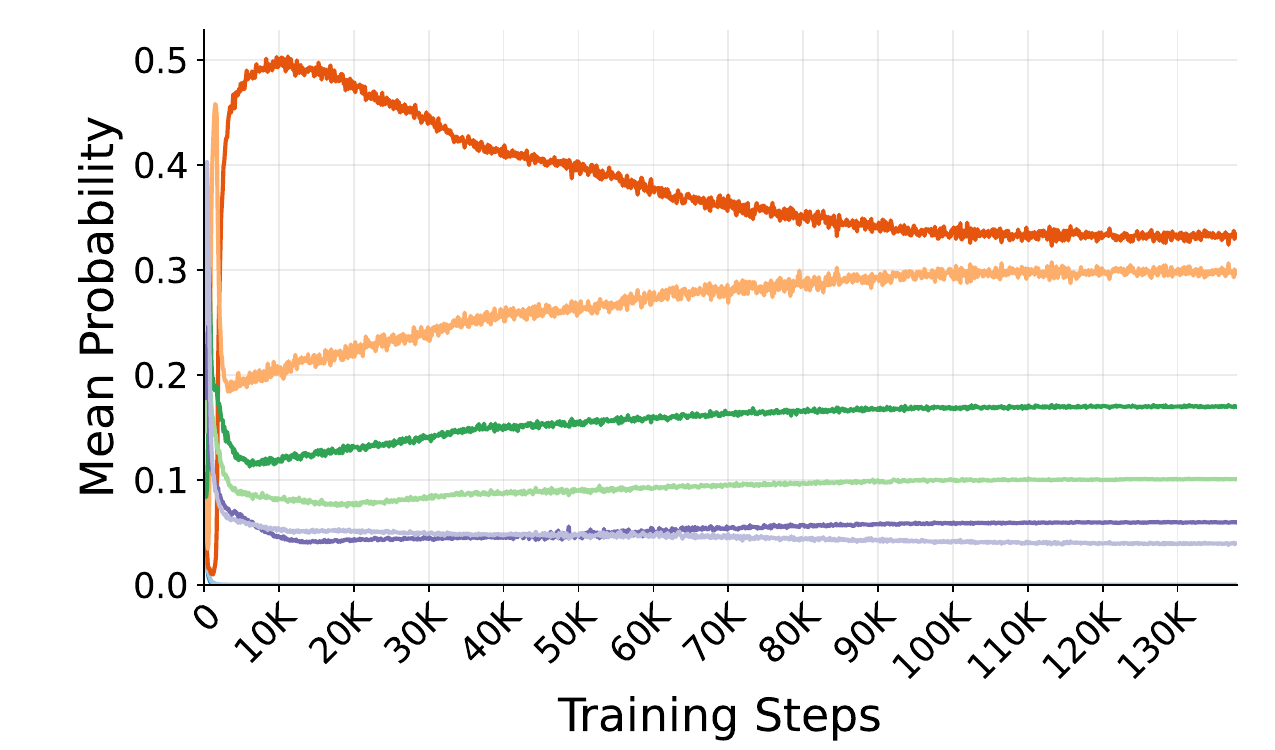}
    \caption{Reacher.}
  \end{subfigure}
  \hfill
  \begin{subfigure}[t]{0.32\linewidth}
    \centering
    \includegraphics[width=\linewidth]{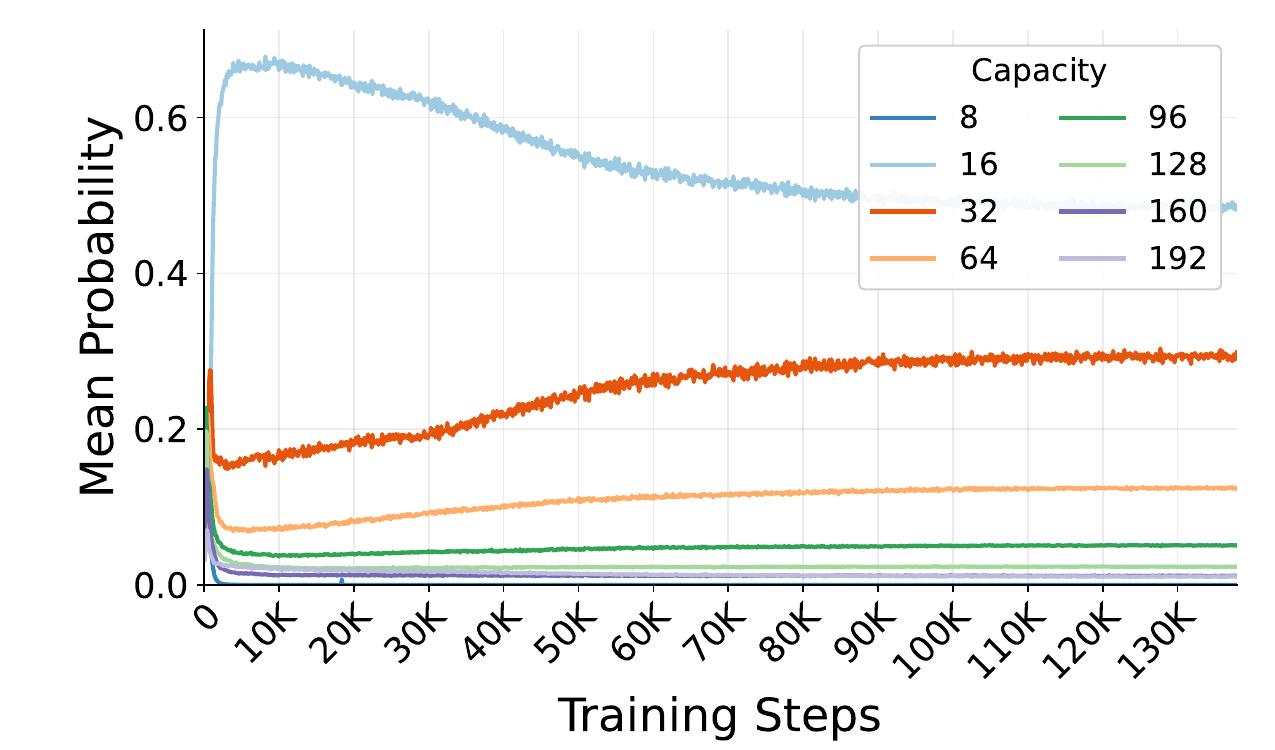}
    \caption{OGBench-Cube.}
  \end{subfigure}
  \caption{Mean capacity probability over training for the ViT-Tiny \MN{} models with $d_{\max}=192$ for one seed.  Each curve is the capacity probability averaged over the training batch at each logged step.}
  \label{fig:capacity-dynamics}
\end{figure*}

\subsection{Mixed-dataset training: different capacities for different datasets}
\label{sec:mixed_training}
As a proof of concept, we test whether one capacity network can learn dataset-dependent capacities when a single model is trained jointly on DMC and TwoRoom. Figure~\ref{fig:mixed-capacity-dynamics} shows that the selector learns distinct distributions for the two datasets. DMC shifts toward a broad high-capacity distribution whose largest component is $k=64$, while retaining substantial mass on $k\in\{96,128,160\}$.  TwoRoom instead concentrates on the smaller $k\in\{16,32\}$ prefixes.  At the terminal test evaluation, DMC obtains $85.17\pm0.73\%$ success with mean selected test-time prefix length $\E[K_{\text{plan}}]=63.95\pm0.05$, whereas TwoRoom obtains only $36.17\pm1.01\%$ success with $\E[K_{\text{plan}}]=21.33\pm5.33$ (mean $\pm$ SEM over three seeds).  Thus, the experiment provides a simple demonstration that the capacity network can assign different capacities to different data while maintaining strong DMC performance.  However, TwoRoom performance degrades severely, so further work is required to preserve performance on both datasets during joint training.

\begin{figure*}[t]
  \centering
  \includegraphics[width=\linewidth]{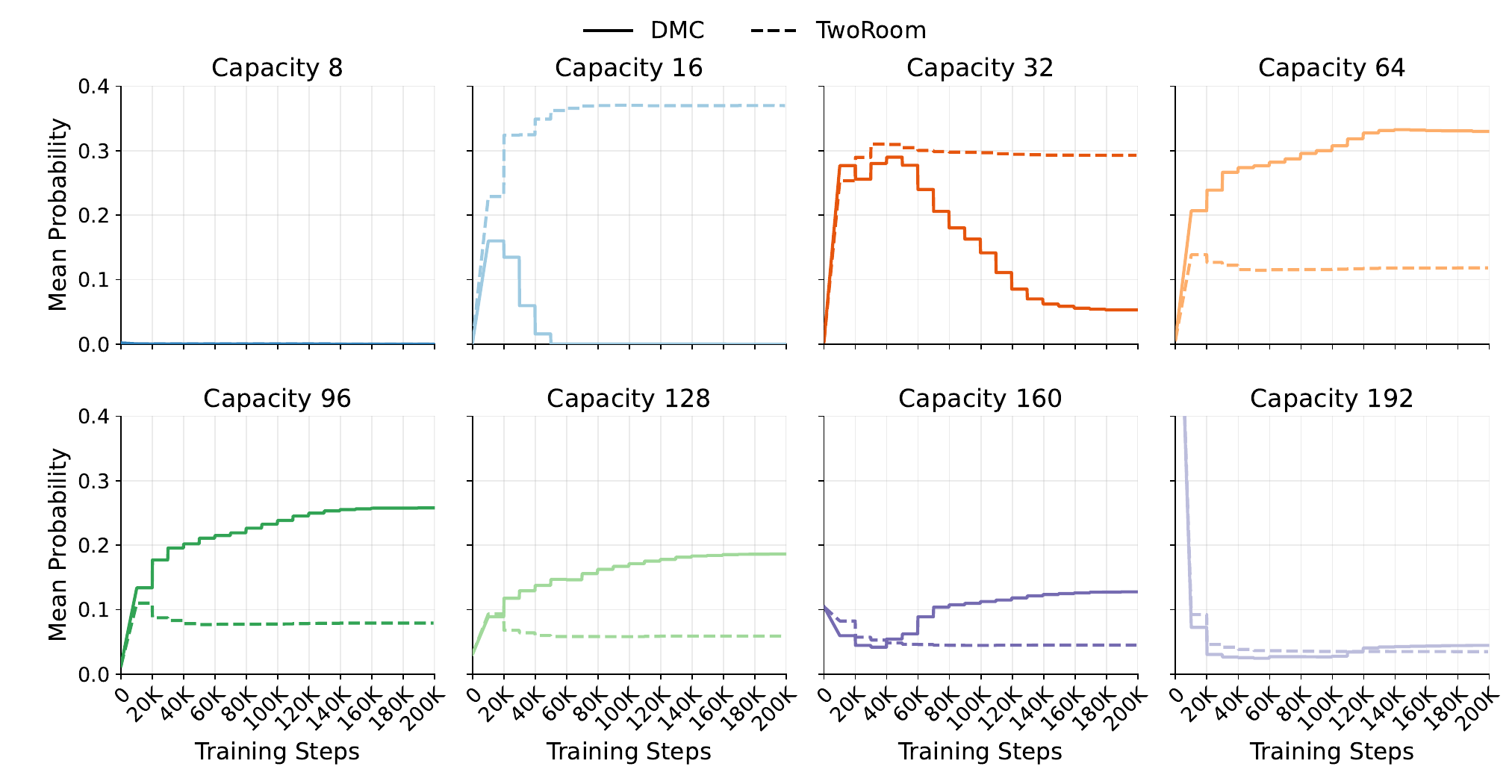}
  \caption{Dataset-conditioned capacity dynamics for joint DMC--TwoRoom training, averaged over three seeds at fixed steps.  Each panel corresponds to one capacity; solid and dashed curves show the mean selector probability for DMC and TwoRoom, respectively.  For better visibility, the vertical range is capped at $0.4$, thereby omitting the initial probability $0.837$ of capacity $k=192$.}
  \label{fig:mixed-capacity-dynamics}
\end{figure*}

\subsection{Reacher Capacity Allocation}
\label{sec:reacher_capacity}
Figure~\ref{fig:reacher-rollout-capacity-allocation} examines
per-episode capacity allocation for ViT-Tiny \MN{} models with
$d_{\max}=192$, evaluated on the same $200$ Reacher test episodes across three training seeds. For seeds $3072$ and $3073$, the selector tends to assign $K=64$ to episodes with smaller physical start--goal gaps. As the distance between the start and goal encoder representations increases, the probability ratio $q_\psi(K=64\mid\rvw)/q_\psi(K=32\mid\rvw)$ tends to decrease, favoring $K=32$.
Seed $3074$, however, selects $K=64$ for all episodes, indicating that the allocation pattern depends on the training seed, although all selections remain within the two neighboring supported capacities, $32$ and $64$. Within each seed and selected-capacity group, failures also tend to be more frequent for episodes with larger physical start--goal arm-tip gaps.

\begin{figure*}[p]
  \centering
  \captionsetup[subfigure]{skip=2pt}
  \begin{subfigure}[t]{0.6\linewidth}
    \centering
    \includegraphics[width=\linewidth]{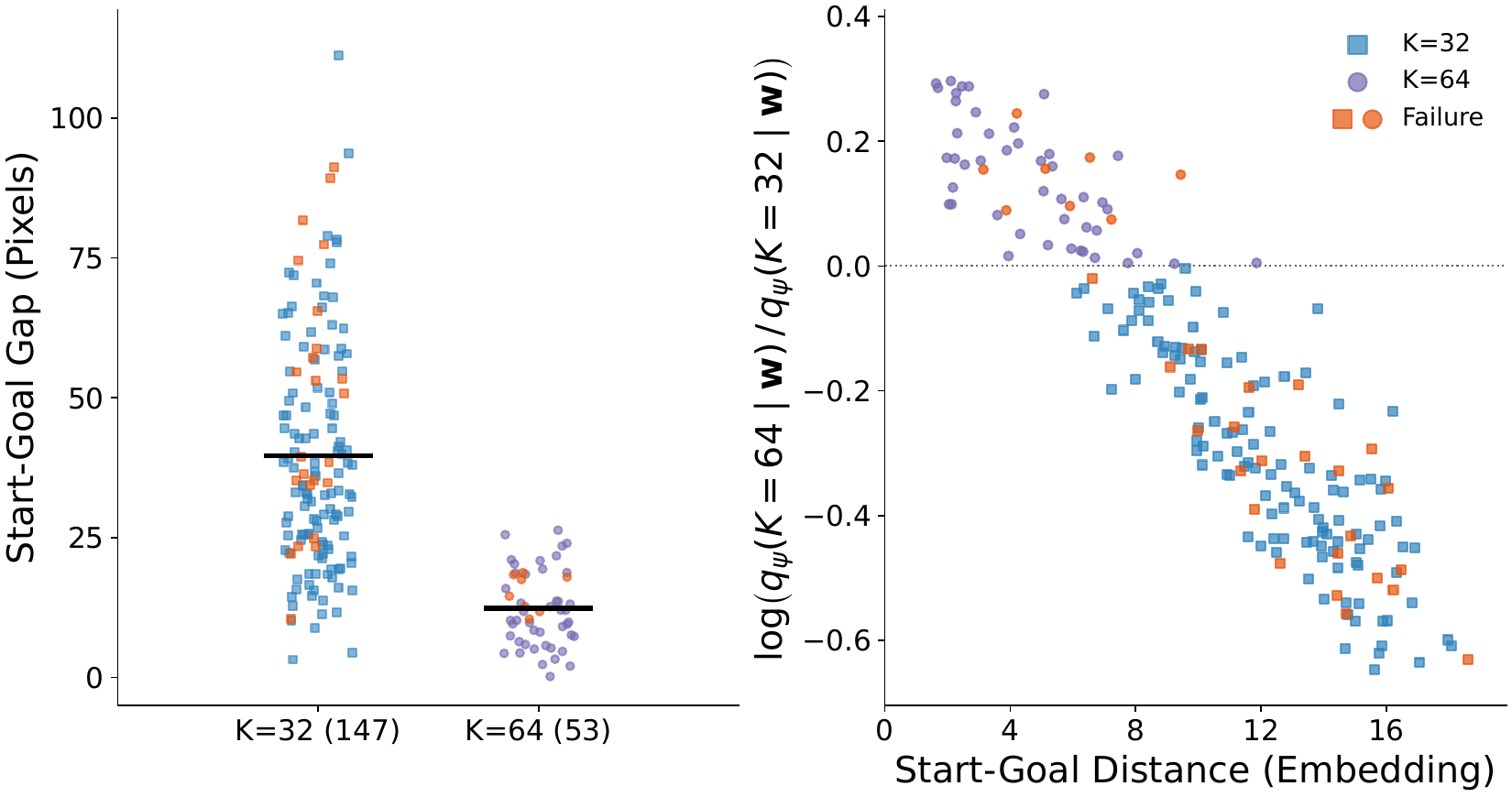}
    \caption{Training seed $3072$}
    \label{fig:reacher-capacity-seed3072}
  \end{subfigure}
  \begin{subfigure}[t]{0.6\linewidth}
    \centering
    \includegraphics[width=\linewidth]{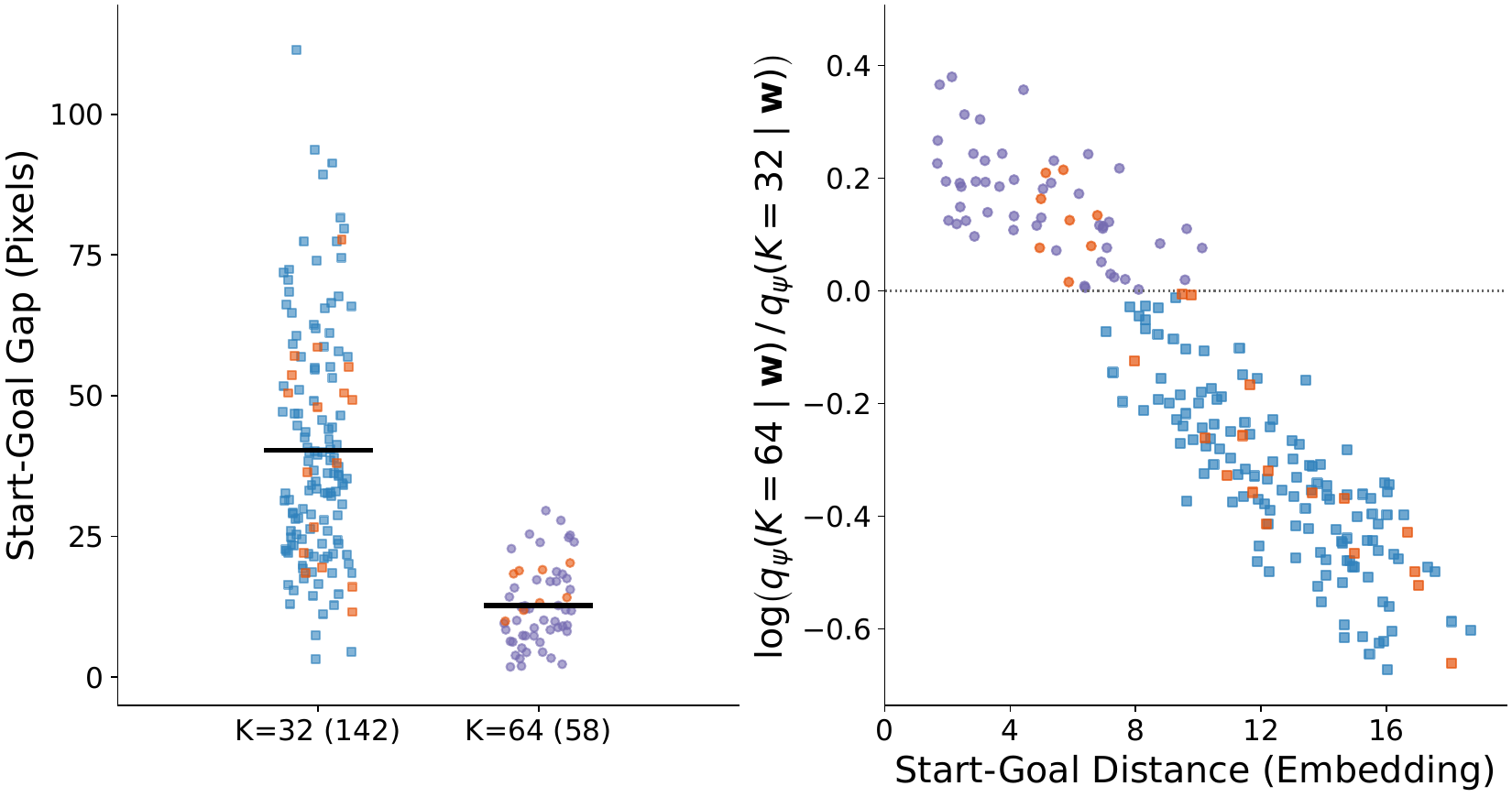}
    \caption{Training seed $3073$}
    \label{fig:reacher-capacity-seed3073}
  \end{subfigure}\par\smallskip
  \begin{subfigure}[t]{0.6\linewidth}
    \centering
    \includegraphics[width=\linewidth]{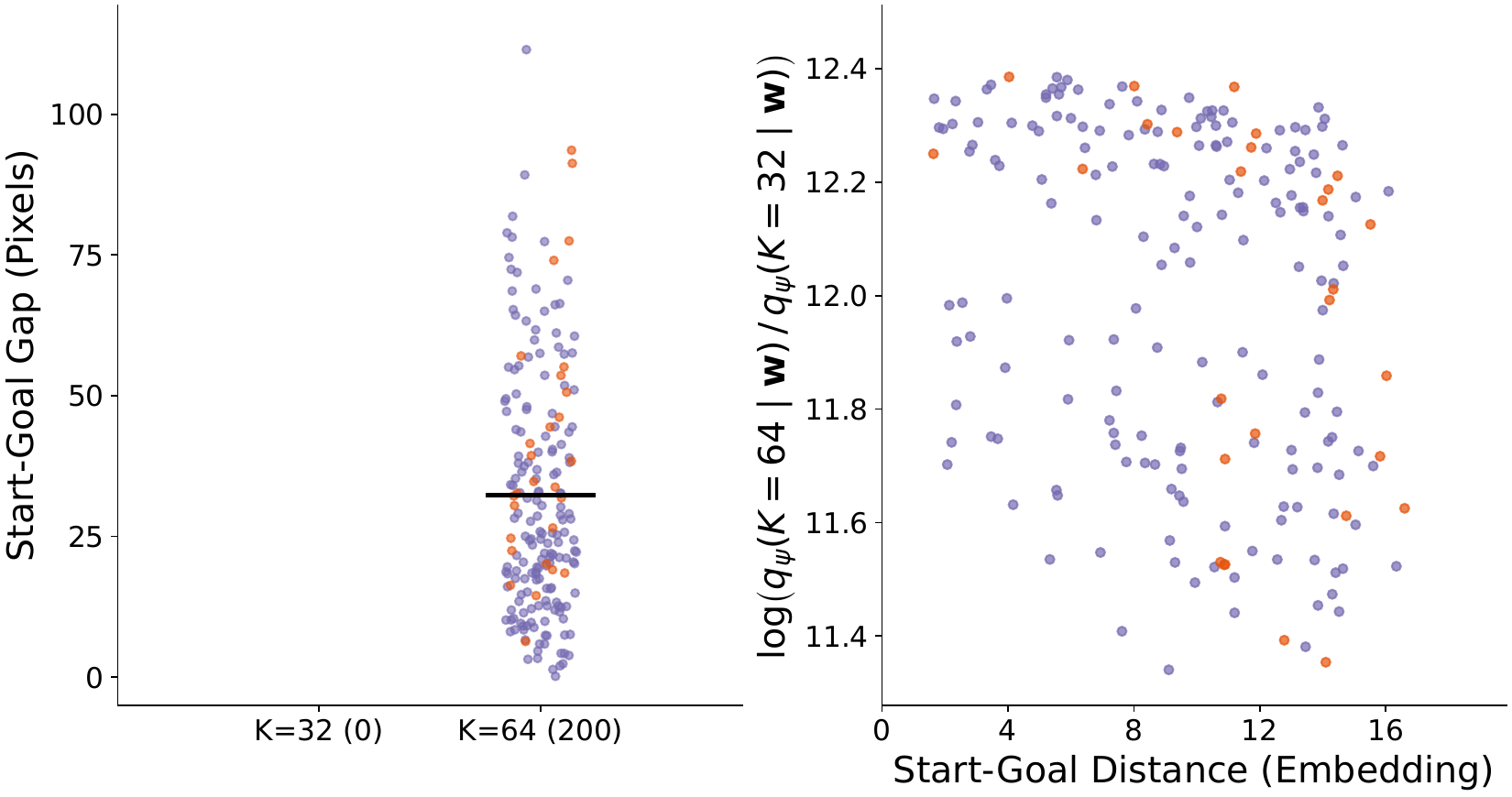}
    \caption{Training seed $3074$}
    \label{fig:reacher-capacity-seed3074}
  \end{subfigure}
  \caption{Reacher rollout capacity allocation for ViT-Tiny \MN{} with $d_{\max}=192$. Left: Euclidean start--goal arm-tip displacement in pixels, grouped by selected capacity; black bars indicate means and parentheses give episode counts. Right: the logged selector log probability ratio versus Euclidean distance between the start and goal encoder CLS tokens. Squares denote $K=32$ and circles $K=64$; blue and purple indicate successful rollouts, respectively, and orange indicates failures.}
  \label{fig:reacher-rollout-capacity-allocation}
\end{figure*}

%\textbf{Ablations.} Polynomial degree.
\subsection{Example Rollouts}
\label{sec:examples}
Figure~\ref{fig:example-rollouts} compares two matched test episodes \textit{randomly selected} per dataset for ViT-Tiny models trained with seed 3072, using $d_{\max}=192$ for \MN{} and the best-$d$ LeWM alternative. We select episodes on which \MN{} meets the environment's success criterion while LeWM does not, with both methods starting from the same state and targeting the same goal. These examples illustrate differences in goal-directed behavior between the two approaches. 

\begin{figure}[p]
  \centering
  \includegraphics[width=\linewidth,height=0.80\textheight,keepaspectratio]{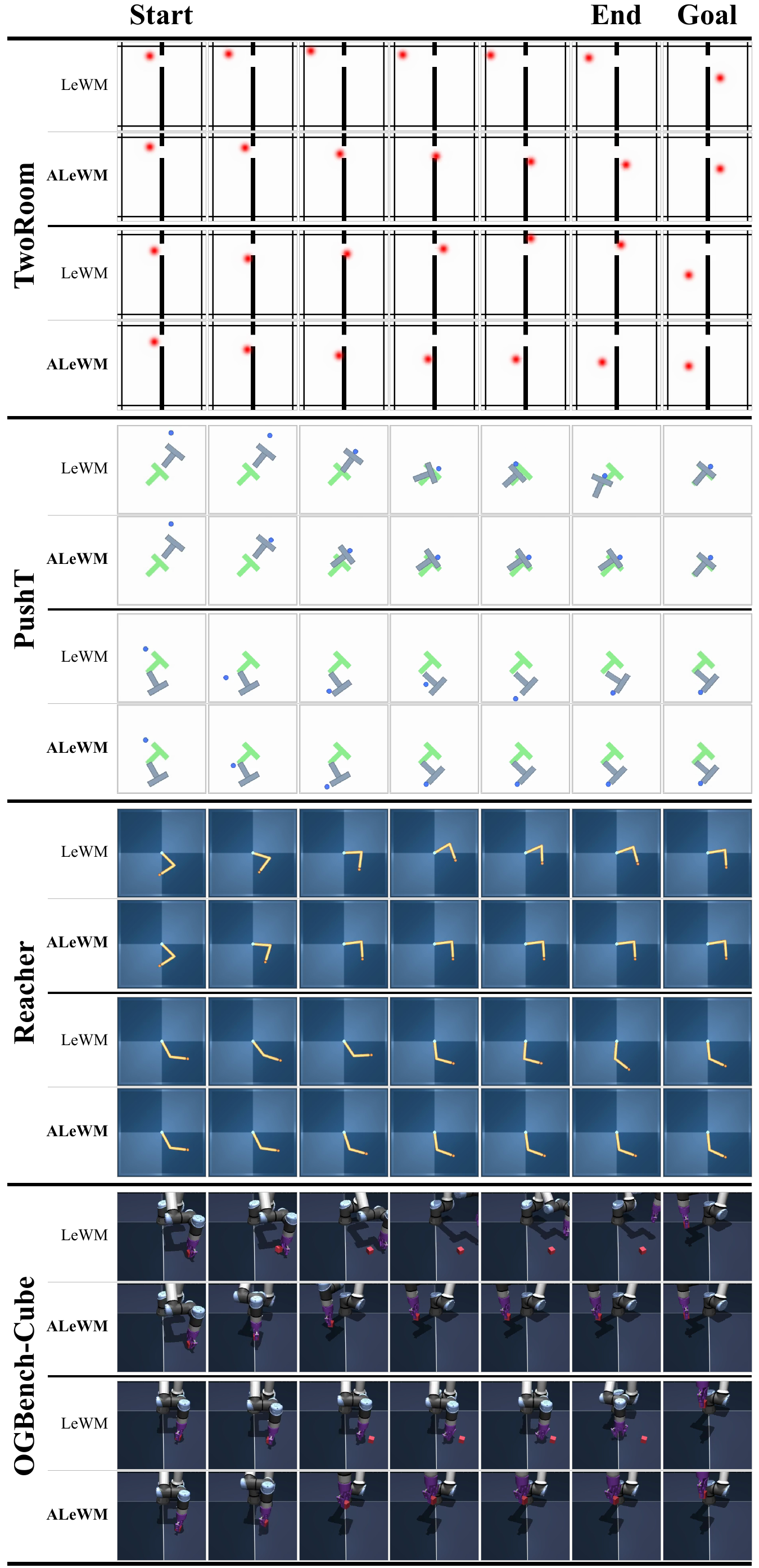}
  \caption{Qualitative comparison on two matched test episodes per dataset, selected for \MN{} success and LeWM failure. Each pair places LeWM above \MN{}. Columns show the initial observation, four approximately uniformly spaced intermediate observations, the final recorded observation, and the shared goal. Images are environment observations under model-planned
  actions.}
  \label{fig:example-rollouts}
\end{figure}

\section{Implementation Details}
\label{sec:implementation_details}
\subsection{Toy Example}
\label{sec:toy-implementation}

\paragraph{Problem and data.}
The state comprises two independent controlled damped oscillators,
$\rvx_t=(p^1_t,v^1_t,p^2_t,v^2_t)$.  For oscillator $m\in\{1,2\}$,
semi-implicit Euler integration gives
\begin{equation}
v^m_{t+1}=v_t^m+\Delta t[-\gamma v_t^m-\kappa(p_t^m-a_t^m)]
+\epsilon_t^m,\qquad
p^m_{t+1}=p_t^m+\Delta t\,v^m_{t+1}+\eta_t^m.
\end{equation}
We set $\Delta t=0.2$, $\gamma=0.25$, $\kappa=1.03887957$,
$\epsilon_t^m\sim\N(0,0.02730495^2)$, and
$\eta_t^m\sim\N(0,0.01365248^2)$.  Each component of the action is sampled
uniformly from $[-1.36524771,1.36524771]$ and held for four transitions.  We
run 256 burn-in transitions, then record nine states and eight actions per
trajectory.  The time step and damping yield stable, underdamped dynamics. The non-round constants reflect a scale calibration that gives position and velocity approximately unit stationary variance under the four-step action-hold policy, while retaining small process noise. The burn-in reduces dependence on the zero initial state. The train, validation, and test splits contain 5000, 1000, and
1000 trajectories.

The ten-dimensional observation contains four strong coordinates produced by an invertible nonlinear wave map followed by additive noise, and six weak coordinates of the form $a_o\sin(\rvw_o^\top\rvx_t+\phi_o)+\epsilon_{t,o}$,
for $o\in\{1,\ldots,6\}$.
The fixed dense vectors $\rvw_o$ mix all four factors and form
a rank-four matrix when stacked as rows. The fixed parameters
$\phi_o$ and $a_o$ specify the phase shift and noiseless signal amplitude, respectively. Measurement noise is independent, zero-mean Gaussian across coordinates, time steps, and trajectories, with standard deviations $0.3$ for weak coordinates and $0.05$ for strong coordinates. The weak coordinates therefore add nonlinear redundant measurements, not independent dynamical factors.

\paragraph{Models and optimization.}
We use an encoder network with two hidden layers of width 64, while the embedding projector network, the predictor, and \MN{}'s capacity network each have one hidden layer of width 64 with SiLU activations.  \MN{} uses maximum dimension eight and supports every prefix
$K=1,\ldots,8$; LeWM uses the fixed dimension specified by each baseline run.
We train for 200 epochs with batch size 256 using AdamW, learning rate
$10^{-3}$, weight decay $10^{-4}$, and gradient-norm clipping at $1.0$.  SIGReg
uses 17 knots and 64 random projections.  All data, initialization, sampling,
SIGReg, and diagnostic random streams use fixed seeds for consistency and reproducibility.

\paragraph{Metrics.}
All reported metrics use the test split.  For linear state recovery, a multivariate ridge probe with coefficient $10^{-6}$ is fitted on standardized training pairs for every prefix and time, $\rvs_t^{1:K}$, and applied unchanged to test
pairs.  We report the arithmetic mean of the four factor-wise coefficients of determination,
\begin{equation}
R_K^2=\frac{1}{4}\sum_{r=1}^4\left(1-
\frac{\sum_n(x_{n,r}-\hat x_{n,r}^{(K)})^2}
{\sum_n(x_{n,r}-\bar x_r)^2}\right),
\end{equation}
where $n$ indexes all test trajectory--time pairs,
$\bar{x}_r$ is the mean of factor $r$ over these pairs, and
$\hat{x}_{n,r}^{(K)}$ is the prediction from the corresponding
latent prefix $\rvs_t^{1:K}$ using the ridge probe fitted on
training data.
For the latent spectrum, we eigendecompose the empirical covariance (with divisor equal to sample
count) of all unmasked test embeddings, normalize its eigenvalues as
$\nu_j=\lambda_j/\sum_\ell\lambda_\ell$, and report
$d_{\mathrm{eff}}=\exp(-\sum_j\nu_j\log\nu_j)$.

For \MN{}, let
$\gamma_{b,j}=\Pr_{q_\psi}(K\geq j\mid\rvw_b)$ be the posterior survival
probability of coordinate $j$ for sequence $b$.  Its soft masked-mixture
variance over $B$ sequences and $T$ encoded time points is
\begin{equation}
v_j^{\mathrm{mix}}=
\frac{1}{BT}\sum_{b,t}\gamma_{b,j}(s_{b,t}^j)^2-
\left(\frac{1}{BT}\sum_{b,t}\gamma_{b,j}s_{b,t}^j\right)^2,
\end{equation}
which we compare with the prior survival
$\Pr_{\pi_0}(K\geq j)$ in Fig. \ref{fig:toy-oscillator}(d).  %Finally, whitened Procrustes uses the ordered
%four-coordinate \MN{} prefix.  Training embeddings and states are whitened,
%an orthogonal map $R^\star=\arg\min_{R^\top R=I}\|Y_wR-X_w\|_F^2$ is fitted,
%and the same training means, whiteners, and rotation are applied to test data.
\paragraph{Whitened Procrustes alignment.}
We examine whether the first four latent coordinates capture the
four-dimensional physical state up to a linear change of coordinates.
For each training trajectory--time pair $n$, let
$\rvy_n=\rvs_t^{1:4}$ denote the latent prefix and $\rvx_n$
the corresponding true state. Stacking their transposes as rows
gives matrices
$\rmY^{\mathrm{tr}},\rmX^{\mathrm{tr}}
\in\sR^{n_{\mathrm{tr}}\times4}$,
where $n_{\mathrm{tr}}$ is the number of training pairs.

We compute the training mean vectors $\boldsymbol{\mu}_y,\boldsymbol{\mu}_x$
and sample covariance matrices $\rmSigma_y,\rmSigma_x$.
Let $\rmY_c^{\mathrm{tr}}$ and $\rmX_c^{\mathrm{tr}}$
contain the centered rows
$(\rvy_n-\boldsymbol{\mu}_y)^\top$ and $(\rvx_n-\boldsymbol{\mu}_x)^\top$.
We apply regularized whitening:
\begin{equation}
\begin{aligned}
\rmY_w^{\mathrm{tr}}
&=\rmY_c^{\mathrm{tr}}
  (\rmSigma_y+\lambda_w\rmI_4)^{-1/2},\\
\rmX_w^{\mathrm{tr}}
&=\rmX_c^{\mathrm{tr}}
  (\rmSigma_x+\lambda_w\rmI_4)^{-1/2},
\end{aligned}
\end{equation}
where $\lambda_w=10^{-6}$ stabilizes the inverse.
Whitening normalizes coordinate scales and correlations, allowing
the alignment to focus on the correspondence between the two
representations.

We then fit the orthogonal transformation
\begin{equation}
\rmR^\star
=\arg\min_{\rmR^\top\rmR=\rmI_4}
\left\|\rmY_w^{\mathrm{tr}}\rmR
      -\rmX_w^{\mathrm{tr}}\right\|_F^2.
\end{equation}
This objective minimizes the total squared discrepancy between
paired, whitened representations. If
$(\rmY_w^{\mathrm{tr}})^\top\rmX_w^{\mathrm{tr}}
=\rmU\rmD\rmV^\top$
is its singular value decomposition, a solution is
$\rmR^\star=\rmU\rmV^\top$.
The orthogonality constraint permits rotations and reflections.

Finally, we apply the training means, whitening matrices, and
$\rmR^\star$ unchanged to the test data. The plot compares
$\rmY_w^{\mathrm{te}}\rmR^\star$ with $\rmX_w^{\mathrm{te}}$,
so its coordinates are expressed in whitened space.
Close agreement supports recovery of state information from the
prefix up to a linear transformation; it does not require individual
latent coordinates to correspond directly to individual state factors.

\paragraph{Hyperparameter search.}
We search over MixSIGReg/SIGReg weight in
$\{0.001,0.0025,0.005,0.0075,0.01,0.05,0.1,0.2,0.3,0.5,1,5\}$.  LeWM searches over fixed dimension
$d\in\{2,3,4,5,6,7,8\}$, while the \MN{} sweep searches polynomial prior degree
in $\{-3,-1.5,-0.5, 0, 0.5, 1.5, 3\}$.  The main text reports degree $-1.5$ with MixSIGReg weight $0.01$ for \MN{}, and
SIGReg weights $0.005$, $0.005$, and $0.0075$ for LeWM dimensions $4$, $6$, and $8$, respectively.

\subsection{Main Results}
\label{sec:main-results-implementation}

\paragraph{Data and preprocessing.}
For TwoRoom, PushT, Reacher, and OGBench-Cube, we follow the environments, offline-data collection, and goal-conditioned evaluation protocol of LeWM; Appendices~D--F.1 of \citet{maes2026leworldmodel} provide the remaining environment and collection details. We create a new episode-disjoint partitions. Reacher, OGBench-Cube, and TwoRoom each contain 10,000 episodes, split into 9,700/100/200 train/validation/test episodes. Validation is used for hyper-parameter selection. The available PushT file contains 18,685 episodes, split into 18,385/100/200 episodes. Images are resized to $224\times224$.  Training examples contain four frames:
three context frames and one prediction target.  Each successive pair is
separated by five environment steps, and the intervening actions are flattened
into a five-action block.  We train and evaluate three model seeds, 3072, 3073,
and 3074.  The success rates in
Table~\ref{tab:main-control-results} first average over the 200 held-out test
episodes for each seed and then report the mean and sample SEM over the three
seeds.

\paragraph{World-model architecture.}
We retain the LeWM architecture and refer to
\citet[Appendix~D]{maes2026leworldmodel} for details not specified here.  The
$d_{\max}\in\{96,192\}$ models use a randomly initialized ViT-Tiny/14 image
encoder (12 blocks, 3 attention heads, and encoder width 192), while the
$d_{\max}=384$ models use ViT-Small/14 (12 blocks, 6 heads, and encoder width
384).  The encoder CLS token is mapped to the latent state by a one-hidden-layer
MLP with hidden width 2048 and batch normalization.  The action-conditioned
autoregressive predictor has six transformer blocks, 16 attention heads, an
MLP width of 2048, and dropout 0.1.  The encoded action blocks condition its
transformer blocks through adaptive layer normalization; a second
one-hidden-layer MLP projects the output to the prediction target.

\paragraph{Capacity network.}
\MN{} adds the categorical selector $q_\psi$.  It takes the raw per-frame image
encoder CLS tokens, before the latent projector, and applies four pre-norm
transformer-encoder blocks with eight heads, feed-forward width 768. A final layer
normalization and mean over the frame-token axis are followed by a linear
$C$-class head, whose logits parameterize the supported capacities.  Self
attention without positional embeddings is permutation equivariant and the
mean is permutation invariant.  Consequently, the architecture and parameter
count do not depend on the number of frames: the same network consumes all
four training frames and the two-image initial/goal context at evaluation.
The selector uses only supplied observations: local trajectory frames during training and the available initial and goal images at deployment. Permutation invariance makes their order irrelevant, and mean pooling allows variable frame counts. While adding or removing observations can change the output, in our short-horizon, controlled tasks, we expect capacity requirements to remain broadly similar across these contexts, although this does not establish calibration under the context shift. Longer or less controlled settings are an interesting direction for studying this dependence. Closed-loop planning could also update capacity using newly observed history, whereas our evaluation keeps it fixed per episode. During training, we intentionally conditioned the capacity network on the entire window of frames, and specifically the predicted frames, as it simulates testing environment where we access to goal frame and possibly other future frames. The CLS tokens are detached on the selector branch, so selector gradients do not flow back into the image encoder. The capacity network contains $1,780,805$, $1,781,384$, and $4,740,491$ trainable parameters for $d_{\max}=96,192,384$, respectively, excluding the shared image encoder.

% For encoder-token width $D_e$, feed-forward width $m$, $L$ transformer blocks,
% and $C$ supported capacities, the trainable parameter count is
% $L\left(4D_e^2+2D_em+9D_e+m\right)+2D_e+C(D_e+1)$.
% With $L=4$ and $m=768$, this gives 1,780,805 parameters for
% $d_{\max}=96$ ($D_e=192,C=5$), 1,781,384 for $d_{\max}=192$
% ($D_e=192,C=8$), and 4,740,491 for $d_{\max}=384$
% ($D_e=384,C=11$).  These counts, obtained directly from the instantiated
% PyTorch modules, exclude the image encoder, latent projectors, predictor, and
% action encoder.

\paragraph{Optimization and regularization.}
All models are trained end-to-end for 10 epochs with batch size 128 using
AdamW, learning rate $5\times10^{-5}$, weight decay $10^{-3}$, bfloat16 mixed
precision, and gradient-norm clipping at 1.0.  Training is deterministic for
each seed. SIGReg and MixSIGReg use 1,024 random projections and the Gaussian frequency window $w(\tau)=\exp(-\tau^2/2)$. We truncate the integral to $[-3,3]$ and exploit the even squared characteristic-function discrepancy to evaluate twice the trapezoidal rule on $[0,3]$. The 17 knots are $\tau_j=3j/16$, $j=0,\ldots,16$, with spacing $h=3/16$ and weights $a_0=a_{16}=h$ and $a_j=2h$ for $1\leq j\leq15$, each multiplied by $w(\tau_j)$. The weighted discrepancy is multiplied by batch size $B$ and averaged over projections and frame positions in both regularizers.  For \MN{}, the full next embedding is the prediction target,
the sampled prefix is obtained with the straight-through Gumbel--Softmax
estimator at a constant relaxation temperature of 0.5.  The selector head is initialized to the polynomial
prior.  We set the selector learning rate to
$s_q(5\times10^{-5})$, where the multiplier $s_q$ is a hyper-parameter as reported below.

The LeWM candidate grids contain the fixed latent widths
$\{8,16,32,64,96\}$ for the 96-dimensional block,
$\{8,16,32,64,96,128,160,192\}$ for the 192-dimensional block, and $\{8,16,32,64,96,128,160,192,256,320,384\}$ for the 384-dimensional block. The corresponding \MN{} capacity supports are fixed and set to $\{8,16,32,64,96\}$, $\{8,16,32,64,96,128,160,192\}$, and
$\{8,16,32,64,96,128,160,192,256,320,384\}$, respectively.  Both methods sweep the MixSIGReg/SIGReg coefficient over task-specific subsets of $\{0.09,0.15,0.2,0.3\}$.  \MN{} sweeps the selector learning-rate multiplier in $\{1, 2, 3\}$ and polynomial-prior degree in $\{-1, -0.5, 0, 1, 2, 3\}$.  Table~\ref{tab:main-selected-hparams}
records the exact \MN{} configurations used in
Table~\ref{tab:main-control-results}.

\begin{table}[H]
  \caption{Selected \MN{} hyperparameters for the main control results.
  $\lambda$ is the MixSIGReg coefficient, $s_q$ multiplies the base learning
  rate for $q_\psi$, and $\alpha$ is the polynomial-prior degree.}
  \label{tab:main-selected-hparams}
  \centering
  \setlength{\tabcolsep}{5pt}
  \begin{tabular}{rlrrr}
    \toprule
    $d_{\max}$ & Dataset & $\lambda$ & $s_q$ & $\alpha$ \\
    \midrule
     96 & TwoRoom      & 0.15 & 3 & 3 \\
     96 & PushT        & 0.09 & 2 & $-1$ \\
     96 & Reacher      & 0.30 & 3 & 2 \\
     96 & OGBench-Cube & 0.30 & 3 & 1 \\
    \midrule
    192 & TwoRoom      & 0.30 & 3 & 3 \\
    192 & PushT        & 0.15 & 2 & $-0.5$ \\
    192 & Reacher      & 0.30 & 3 & $-0.5$ \\
    192 & OGBench-Cube & 0.20 & 3 & 1 \\
    \midrule
    384 & TwoRoom      & 0.20 & 2 & 3 \\
    384 & PushT        & 0.09 & 1 & $-0.5$ \\
    384 & Reacher      & 0.20 & 2 & 1 \\
    384 & OGBench-Cube & 0.15 & 2 & 0 \\
    \bottomrule
  \end{tabular}
\end{table}

\begin{table*}[p]
  \caption{Per-seed planning-capacity counts for the selected \MN{}
  configurations in Table~\ref{tab:main-control-results}.  An entry $k{:}\ n$
  means that $K_{\mathrm{plan}}=k$ for $n$ episodes; capacities that were
  never selected are omitted.  Test contains 200 episodes per seed.}
  \label{tab:main-planning-capacity-counts}
  \centering
  \scriptsize
  \setlength{\tabcolsep}{6pt}
  \renewcommand{\arraystretch}{0.92}
  \begin{tabular}{rrrl}
    \toprule
    Backbone ($d_{\max}$) & Dataset & Seed & $K_{\mathrm{plan}}$: \#rollouts \\
    \midrule
    \multirow{12}{*}{ViT-Tiny (96)}
      & \multirow{3}{*}{TwoRoom} & 3072 & $16{:}\ 200$ \\
      & & 3073 & $8{:}\ 200$ \\
      & & 3074 & $16{:}\ 200$ \\
    \cmidrule(lr){2-4}
      & \multirow{3}{*}{PushT} & 3072 & $32{:}\ 200$ \\
      & & 3073 & $32{:}\ 200$ \\
      & & 3074 & $32{:}\ 200$ \\
    \cmidrule(lr){2-4}
      & \multirow{3}{*}{Reacher} & 3072 & $32{:}\ 200$ \\
      & & 3073 & $16{:}\ 200$ \\
      & & 3074 & $16{:}\ 200$ \\
    \cmidrule(lr){2-4}
      & \multirow{3}{*}{OGBench-Cube} & 3072 & $16{:}\ 200$ \\
      & & 3073 & $16{:}\ 200$ \\
      & & 3074 & $16{:}\ 200$ \\
    \midrule
    \multirow{12}{*}{ViT-Tiny (192)}
      & \multirow{3}{*}{TwoRoom} & 3072 & $8{:}\ 200$ \\
      & & 3073 & $8{:}\ 200$ \\
      & & 3074 & $8{:}\ 200$ \\
    \cmidrule(lr){2-4}
      & \multirow{3}{*}{PushT} & 3072 & $64{:}\ 200$ \\
      & & 3073 & $32{:}\ 197,\ 64{:}\ 3$ \\
      & & 3074 & $64{:}\ 200$ \\
    \cmidrule(lr){2-4}
      & \multirow{3}{*}{Reacher} & 3072 & $32{:}\ 147,\ 64{:}\ 53$ \\
      & & 3073 & $32{:}\ 142,\ 64{:}\ 58$ \\
      & & 3074 & $64{:}\ 200$ \\
    \cmidrule(lr){2-4}
      & \multirow{3}{*}{OGBench-Cube} & 3072 & $16{:}\ 200$ \\
      & & 3073 & $16{:}\ 200$ \\
      & & 3074 & $16{:}\ 200$ \\
    \midrule
    \multirow{12}{*}{ViT-Small (384)}
      & \multirow{3}{*}{TwoRoom} & 3072 & $8{:}\ 196,\ 16{:}\ 4$ \\
      & & 3073 & $8{:}\ 195,\ 16{:}\ 5$ \\
      & & 3074 & $8{:}\ 197,\ 16{:}\ 3$ \\
    \cmidrule(lr){2-4}
      & \multirow{3}{*}{PushT} & 3072 & $64{:}\ 200$ \\
      & & 3073 & $32{:}\ 162,\ 64{:}\ 38$ \\
      & & 3074 & $64{:}\ 200$ \\
    \cmidrule(lr){2-4}
      & \multirow{3}{*}{Reacher} & 3072 & $32{:}\ 200$ \\
      & & 3073 & $32{:}\ 200$ \\
      & & 3074 & $32{:}\ 200$ \\
    \cmidrule(lr){2-4}
      & \multirow{3}{*}{OGBench-Cube} & 3072 & $32{:}\ 200$ \\
      & & 3073 & $32{:}\ 200$ \\
      & & 3074 & $32{:}\ 200$ \\
    \bottomrule
  \end{tabular}
\end{table*}

\paragraph{Planning and evaluation.}
The four control environments use CEM with 300 candidate action
sequences, initial variance 1, 30 refinement iterations, and 30 elites.  The
planning horizon is five action blocks of five actions each.  These tasks
execute five blocks before replanning. Each held-out episode is initialized from an offline trajectory,
uses the observation 25 environment steps later as its goal, and has a
50-step interaction budget.  Standard-task rollouts are evaluated in batches of 50. %The evaluation random seed is 42. 
For \MN{}, $q_\psi$ processes the initial and goal images once per episode,
the modal supported capacity is cached across CEM iterations and replans, and
both recursive dynamics rollout and goal distance use only that active prefix.
LeWM uses its entire fixed-width representation.  These choices match the
``active'' target-dimension setting reported in the main table.

\end{document}